\documentclass[a4,11pt]{article}
\usepackage[verbose=true,letterpaper]{geometry}
\AtBeginDocument{
  \newgeometry{
    textheight=9in,
    textwidth=6.5in,
    top=1in,
    headheight=14pt,
    headsep=25pt,
    footskip=30pt
  }
}

\usepackage{natbib}

\usepackage[utf8]{inputenc} %
\usepackage[T1]{fontenc}    %
\usepackage{hyperref}       %
\usepackage{url}            %
\usepackage{booktabs}       %
\usepackage{amsfonts}       %
\usepackage{nicefrac}       %
\usepackage{microtype}      %
\usepackage{xcolor}         %
\usepackage[inline]{enumitem}

\usepackage{amsthm,amsmath,amsfonts,amssymb,bm,bbm,mathtools}
\newtheorem{theorem}{Theorem}
\newtheorem{corollary}{Corollary}
\newtheorem{lemma}{Lemma}
\newtheorem{proposition}{Proposition}
\newtheorem{assumption}{Assumption}
\theoremstyle{definition}
\newtheorem{definition}{Definition}

\newcommand{\attn}{\mathrm{Attn}}
\newcommand{\tf}{\mathrm{TF}}
\newcommand{\mlp}{\mathrm{MLP}}

\def\vzero{{\bm{0}}}

\def\vtheta{{\bm{\theta}}}
\def\va{{\bm{a}}}
\def\vb{{\bm{b}}}

\def\ve{{\bm{e}}}

\def\vg{{\bm{g}}}
\def\vh{{\bm{h}}}

\def\vu{{\bm{u}}}
\def\vv{{\bm{v}}}
\def\vw{{\bm{w}}}
\def\vx{{\bm{x}}}

\def\vz{{\bm{z}}}
\def\valpha{{\bm \alpha}}

\def\vgamma{{\bm \gamma}}

\def\vepsilon{{\bm \epsilon}}

\def\vtheta{{\bm \theta}}

\def\vphi{{\bm \phi}}

\def\vpsi{{\bm \psi}}

\def\mH{{\bm{H}}}

\def\mK{{\bm{K}}}

\def\mQ{{\bm{Q}}}

\def\mU{{\bm{U}}}
\def\mV{{\bm{V}}}
\def\mW{{\bm{W}}}
\def\mX{{\bm{X}}}

\def\mPsi{{\bm{\Psi}}}

\def\gD{\mathcal{D}}

\def\gF{\mathcal{F}}

\def\gO{\mathcal{O}}

\def\gW{\mathcal{W}}
\def\gX{\mathcal{X}}
\def\gY{\mathcal{Y}}

\newcommand{\E}{\mathbb{E}}
\newcommand{\R}{\mathbb{R}}

\def\1{\mathbbm{1}}

\DeclareMathOperator*{\argmin}{argmin}

\DeclareMathOperator*{\softmin}{softmin}

\providecommand{\abs}[1]{\left\lvert#1\right\rvert}
\providecommand{\norm}[1]{\left\lVert#1\right\rVert}

\newcommand{\stand}{\mathrm{d}}
\renewcommand{\hat}{\widehat}
\renewcommand{\tilde}{\widetilde}

\usepackage[capitalize]{cleveref}
\allowdisplaybreaks

\title{Automatic Domain Adaptation \\by Transformers in In-Context Learning}

\author{%
  Ryuichiro Hataya \\
  RIKEN AIP \\
  \texttt{ryuichiro.hataya@riken.jp} \\
  \and
  Kota Matsui \\
  Nagoya University \\
  \texttt{matsui.kota.x3@f.mail.nagoya-u.ac.jp} \\
  \and
  Masaaki Imaizumi \\
  The University of Tokyo \& RIKEN AIP \\
  \texttt{imaizumi@g.ecc.u-tokyo.ac.jp}
}
\date{}

\begin{document}

\maketitle

\begin{abstract}
Selecting or designing an appropriate domain adaptation algorithm for a given problem remains challenging.
This paper presents a Transformer model that can provably approximate and opt for domain adaptation methods for a given dataset in the in-context learning framework, where a foundation model performs new tasks without updating its parameters at test time.
Specifically, we prove that Transformers can approximate instance-based and feature-based unsupervised domain adaptation algorithms and automatically select an algorithm suited for a given dataset.
Numerical results indicate that in-context learning demonstrates an adaptive domain adaptation surpassing existing methods.
\end{abstract}

\section{Introduction}\label{sec:introduction}

Domain adaptation provides a methodology for transferring ``knowledge'' obtained in one domain to another related domain~\citep{ben2010theory}.
In particular, the \emph{unsupervised domain adaptation (UDA)} setting allows for training of the predictive model using the rich labeled data in the source domain, when there is no labeled data in the target domain.
Such methods are vital in the areas where sample sizes are inherently small, such as in the medical field, and numerous studies have been reported focusing on domain adaptation for medical image analysis, for example~\citep{guan2021domain}.
In this study, we focus in particular on instance-based and feature-based unsupervised  domain adaptation methods.

One of the challenges associated with domain adaptation is selecting and differentiating effective methods. 
There are many existing approaches in domain adaptation, such as the instance-based methods~\citep{dai2007boosting,sugiyama2007covariate, kanamori2009least} and the feature-based methods\citep{daume2009frustratingly,ganin2016domain}. 
However, when each method is effective differs based on data.
Specifically, the instance-based methods~\citep{sugiyama2007covariate,kanamori2009least} are effective when there exists a valid density ratio between covariates of the source and target domains, whilst the feature-based methods are effective if we can find a domain-invariant representation on a common feature space between the domains. 
To properly select an appropriate approach to the given data, it is essential to assess whether they meet the specific conditions of each method. 

There have been some attempts to automatically select appropriate UDA algorithms.
For example, \emph{learning to transfer (L2T)}~\citep{ying2018transfer} employs a meta-learning approach to predict appropriate features to transfer for a given pair of source and target domains.
In addition to L2T, \emph{meta-transfer learning}~\citep{SunLCS2019MTL} and \emph{CNAPS}~\citep{requeima2019fast}, which learn functions to adjust weights of trained neural nets in a task-adaptive manner, have been proposed, and thus algorithm selection based on meta-learning have achieved some success. 
However, these existing methods have been developed only for selection within specific frameworks, such as feature-based domain adaptation.
That is, the possibility of cross-framework algorithm selection, such as ``choosing which of the instances and features to transfer'' was unknown.

As a method for developing versatile algorithms, the in-context learning ability of foundation models has garnered significant attention. 
In-context learning, a form of meta-learning in foundational models, allows models to adapt to new tasks without updating their parameters.
In particular, the capabilities of Transformers in in-context learning have been intensively surveyed: for example, \cite{bai2023transformers} showed that Transformers can learn algorithms such as gradient descent, as well as their optimal selection methods.
Such in-context learning abilities of Transformers have been validated from both experimental and theoretical aspects \citep{garg2022can,li2023transformers,von2023transformers,akyurek2022learning,xie2021explanation,zhang2023trained,bai2023transformers,lin2023transformers,ahn2024transformers,raventos2024pretraining}.

In this study, we demonstrate that Transformer models effectively address the challenge in in-context learning. 
Specifically, our analysis reveals that Transformers in the framework (i) solve domain adaptation problems by approximating the main UDA algorithms, and (ii) automatically select suitable methods adaptively to the dataset's characteristics. 
The results indicate that Transformers can, in context, not only implement various transfer learning methods but also possess the capability to appropriately select among them.
These findings suggest that by accurately choosing a combination of these methods based on the dataset, performance can be enhanced beyond what is achievable by applying the methods individually.

Our specific contributions are as follows:
\begin{enumerate}[leftmargin=*]
    \item We demonstrate that the Transformer can approximate the instance-based transfer learning method based on importance weighting with the uLSIF estimator, which utilizes the density ratio to address covariate shift. 
    We prove that Transformers internally learn the density ratio, effectively implementing uLSIF and importance weighting. 
    A technical contribution of this proof is showing how Transformers can approximate the computation of inverse matrices multiplication required for uLSIF.
    \item We show that Transformers can approximate the feature-based transfer learning method DANN, which uses adversarial learning. 
    Within their architecture, Transformers solve the minimax optimization using a dual-loop mechanism. A significant technical contribution is developing the components that compute this dual-loop.
    \item We numerically demonstrate that Transformers trained with in-context learning perform better in domain adaptation tasks than separate existing methods. This empirically proves that Transformers can adaptively handle complex transfer learning tasks.
\end{enumerate}

\section{Preliminary}\label{sec:preliminary}

\subsection{Unsupervised Domain Adaptation}\label{sec:uda_preliminary}

\begin{description}[leftmargin=*]
    \item[Setup] 
    We consider the \emph{Unsupervised domain adaptation (UDA)} problem with two domains.
    Let $\gX\subset\R^d$ be a compact input space and $\gY$ be an output space.
    Then, the source and the target domains, $P_S$ and $P_T$, are distributions on $\gX\times \gY$, let their density functions be $p_S$ and $p_T$, and their marginal distributions on $\gX$ be $P_S^X$ and $P_T^X$.
    Denote $\gD_S=\{(\vx^S_i, y^S_i)\}_{i=1}^n$ and $\gD_T = \{\vx^T_i\}_{i=1}^{n^{\prime}}$ for the source labeled data and target unlabeled data, respectively and let $N\coloneqq n+n'$.
    Given a loss function $\ell: \gY\times\gY\to\R_{\geq 0}$ and a hypothesis space $\gF$, UDA aims to minimize the target risk $\argmin_{f\in\gF}R_{T}(f)$, where $R_{T}(f)\coloneqq\E_{(\vx, y)\sim P_T}[\ell(f(\vx), y)]$ is the target risk, without observing labels of the data from the target domain.
    To this end, UDA methods aim to match the source and target distributions in the feature space, categorized into instance-based and feature-based approaches.

    \item[Instance-based Methods]
    Instance-based approaches reweight instances from the source domain to minimize the target risk, based on the following transferability assumption called \emph{covariate shift}.
    \begin{assumption}[Covariate Shift, \cite{shimodaira2000improving}]
        Suppose that the distribution of $y$ conditioned on the input $\vx$ is the same in the source and target domains, i.e.,
        \begin{equation}
            p_S(y|\vx)=p_T(y|\vx).
        \end{equation}
    \end{assumption}
    Under this assumption, $R_T(f)=\E_{(\vx, y)\sim P_S}(q(\vx)\ell(f(\vx), y))$ holds, where $q(\vx)=p_T(\vx)/p_S(\vx)$ is the density ratio between the marginal distributions of source and target.
    Then, \emph{Importance-weighted learning (IWL)}~\cite{sugiyama2012density,kimura2024a} obtains $\argmin_{f\in\gF} R_T(f)$ in two steps:
    \begin{enumerate*}[]
        \item Empirically learn an estimator of the density ratio $\hat{q}$ using $\{\vx_i^S\}_{i=1}^n$ and $\gD_T$, and then,
        \item minimize the weighted empirical risk %
        $\hat{R}_T(f) = \frac{1}{n} \sum_{i=1}^n \hat{q}(\vx^S_i) \ell(f(\vx^S_i), y^S_i)$ consists of data from the source domain. 
    \end{enumerate*}

    Here, we introduce \emph{unconstrained Least-Squares Importance Fitting (uLSIF)}~\cite{kanamori2009least}, a method for density-ratio estimation.
    Firstly, the density ratio is estimated by the linear basis function model 
    \begin{equation}
        \hat{q}_\valpha(\vx) = \valpha^\top\vphi(\vx) = \sum_{j=1}^J\alpha_j\phi_j(\vx),
    \end{equation}
    where $\valpha=[\alpha_1,\dots,\alpha_J]$ and $\vphi=[\phi_1,\dots,\phi_J]$ are parameter and feature vectors respectively, where $\phi_j:\gX\mapsto\R$  are feature maps for an input $\vx$.
    uLSIF estimates the weight $\valpha$ to directly minimize the following squared error:
    \begin{equation}\label{eq:ulsif_alpha_estimation}
        L(\valpha)=\frac{1}{2}\int_{\gX}(\hat{q}_\valpha(\vx)-q(\vx))^2 p_S(\vx)\stand\vx
        =\int_{\gX}
        \left\{
        \frac{1}{2}\hat{q}_\valpha(\vx)^2p_S(\vx)
        -\hat{q}_\valpha p_T(\vx)
        +\frac{1}{2}q(\vx)^2p_T(\vx)
        \right\}\stand\vx. 
    \end{equation}
    As the third term of the RHS is constant to $\valpha$, we can estimate $\valpha$ as $\hat{\valpha}=\argmin_{\valpha\geq 0} \hat{L}(\valpha)$ such that
    \begin{equation} \label{eq:iwl_alpha_estimation}
        \hat{L}(\valpha)
        =\frac{1}{2n}\sum_{i=1}^n(\hat{q}_\valpha(\vx))^2-\frac{1}{n'}\sum_{i=1}^{n'}\hat{q}_\valpha(\vx)+\frac{\lambda}{2}\norm{\valpha}^2_2
        =\frac{1}{2}\valpha^\top\mPsi\valpha-\vpsi^\top\valpha+\frac{\lambda}{2}\valpha^\top\valpha,
    \end{equation}
    where $\lambda>0$  is a regularization parameter, $\Psi_{jj'}\coloneqq\frac{1}{n}\sum_{i=1}^n\phi_j(\vx_i^S)\phi_{j'}(\vx_i^S)$ and $\psi_j\coloneqq\frac{1}{n'}\sum_{i=1}^{n'}\phi_j(\vx_i^T)$.
    Then, if $\gF$ is a space of linear model, we can obtain a classifier with an empirical importance-weighted problem: 
    \begin{align}\label{eq:iwl_weight_estimation}
        \vw^*=\argmin_{\vw\in\R^J}\hat{R}_T(\vw), ~ \hat{R}_T(\vw) := \sum_{(\vx, y)\in \gD_S}\hat{q}_{\hat{\valpha}}(\vx)\ell(\vw^\top\vphi(\vx), y).
    \end{align}
    Then, we obtain a classifier defined as
    \begin{align}
        \hat{f}^{\mathrm{IWL}}(\vx) := (\vw^*)^\top \phi(\vx).
    \end{align}
    
    \item[Feature-based Methods]
    A feature-based method learns a domain-invariant feature map $\phi: \gX\to\gX'$, where $\gX'$ is a feature space, so that $R_T(f')\approx\E_{(\vx, y)\sim P_S}\ell(f'(\phi(\vx)), y)$, where $f': \gX'\to\gY$.
    The \emph{domain adversarial neural network (DANN)}~\citep{ganin2016domain} is a typical method to achieve this by using adversarial learning.
    DANN consists of three modules, a feature extractor $f_F: \gX\to\gX'$, a label classifier $f_L:\gX'\to[0, 1]$, and a domain discriminator $f_D:\gX'\to[0, 1]$, parameterized by $\vtheta_F, \vtheta_L, \vtheta_D$, respectively.
    Then, a classifier $f_L$ with an invariant feature extractor $f_F$ can be obtained by solving the following minimax problem:
    \begin{align}
        \min_{\vtheta_F, \vtheta_L}\max_{\vtheta_D} L(\vtheta_F, \vtheta_L)-\lambda\Omega(\vtheta_F, \vtheta_D). \label{eq:dann_minimax}
    \end{align}
    where $\lambda > 0$ is a regularization parameter, $L(\vtheta_F, \vtheta_L)\coloneqq\sum_{(\vx, y)\in\gD_S}\gamma(f_L\circ f_F(\vx), y)$ is label classification loss and $\Omega(\vtheta_F, \vtheta_D)\coloneqq\frac{1}{n}\sum_{(\vx, y)\in\gD_S}\gamma(f_D\circ f_F(\vx), 0)+\frac{1}{n'}\sum_{(\vx, y)\in\gD_T}\gamma(f_D\circ f_F(\vx), 1)$ is domain classification loss.
    Here, $\gamma(p, q)\coloneqq-q\log p-(1-q)\log(1-p)$ for $p, q\in[0, 1]$ is a sigmoid cross entropy function.
    The model parameters are updated by gradient descent with a learning rate of $\eta$ as follows:
    \begin{align}
        \vtheta_F&\gets \vtheta_F-\eta\nabla_{\vtheta_F}\left(L(\vtheta_F, \vtheta_L)-\lambda \Omega(\vtheta_F, \vtheta_D)\right),  \label{eq:dann1}\\
        \vtheta_L&\gets \vtheta_L-\eta\nabla_{\vtheta_L}L(\vtheta_F, \vtheta_L), \label{eq:dann2}\\
        \vtheta_D&\gets \vtheta_D-\eta\lambda\nabla_{\vtheta_D} \Omega(\vtheta_F, \vtheta_D).\label{eq:dann3}
    \end{align}
\end{description}

\subsection{In-context Learning}\label{sec:icl_preliminary}
\begin{description}[leftmargin=*]
\item[Setup.]
In in-context learning, a fixed Transformer observes a dataset $\gD=\{(\vx_i, y_i)\}_{i=1}^N\sim(P)^N$ with pairs of an input $\vx_i$ and its label $y_i$ from a joint distribution $P$ and a new query input $\vx_*$ then predicts $y_*$ corresponding to $\vx_*$.
Different from the standard supervised learning, in in-context learning, the Transformer is pre-trained on other datasets $\gD'$ from different distributions to learn an algorithm to predict $y_*$.
Unlike the standard meta learning, at the inference time of in-context learning, the parameters of a Transformer are fixed.
Our interest is to study the expressive power of a Transformer model for algorithms on a given dataset $\gD$.
    
\item[Transformer.]
Define an $L$-layer Transformer consisting of $L$ Transformer layers as follows.
The $l$th Transformer layer maps an input matrix $\mH^{(l)}\in\R^{D\times N}$ to $\tilde{\mH}^{(l)}\in\R^{D\times N}$ and is composed of a self-attention block and a feed-forward block.
The self-attention block $\attn^{(l)}: \R^{D\times N}\to\R^{D\times N}$ is parameterized by $D\times D$ matrices $\{(\mK^{(l)}_m, \mQ^{(l)}_m, \mV^{(l)}_m)\}_{m=1}^M$, where $M$ is the number of heads, and defined as
\begin{equation}
    \attn^{(l)}(\mX)=\mX+\frac{1}{N}\sum_{m=1}^M \mV^{(l)}_m\mX\sigma((\mQ^{(l)}_m\mX)^\top\mK^{(l)}_m\mX).\label{eq:transformer_attn}
\end{equation}
$\sigma$ denotes an activation function applied elementwisely.

The feed-forward block $\mlp^{(l)}: \R^{D\times N}\to\R^{D\times N}$ is a multi-layer perceptron with a skip connection, parameterized by $(\mW^{(l)}_1, \mW^{(l)}_2)\in \R^{D'\times D}\times \R^{D\times D'}$, such that
\begin{equation}
    \mlp^{(l)}(\mX)=\mX+\mW^{(l)}_2\varsigma(\mW^{(l)}_1\mX), \label{eq:transformer_mlp}
\end{equation}
where $\varsigma$ is an activation function applied elementwisely.
In this paper, we let both $\sigma$ and $\varsigma$ the ReLU function in this paper, following \cite{bai2023transformers}.

In summary, an $L$-layer Transformer $\tf_\vtheta$, parameterized by 
\begin{equation}
    \vtheta=\{(\mK^{(l)}_1, \mQ^{(l)}_1, \mV^{(l)}_1, \dots, \mK^{(l)}_M, \mQ^{(l)}_M, \mV^{(l)}_M, \mW^{(l)}_1, \mW^{(l)}_2)\}_{l=1}^L,
\end{equation} 
is a composition of the abovementioned layers as
\begin{equation}
    \tf_\vtheta(\mX)=\mlp^{(L)}\circ\attn^{(L)}\circ\dots\circ\mlp^{(1)}\circ\attn^{(1)}(\mX).
\end{equation}

In the remaining text, the superscript to indicate the number of layer ${}^{(l)}$ is sometimes omitted for brevity, and the $n$th columns of $\mH^{(l)}, \tilde{\mH}^{(l)}$ are denoted as $\vh^{(l)}_n, \tilde{\vh}^{(l)}_n$.
Following \cite{bai2023transformers}, we define the following norm of a Transformer $\tf_\vtheta$:
\begin{equation}
    \norm{\vtheta}_\tf = \max_{l\in\{1,\dots, L\}}\left\{\max_{m\in\{1,\dots,M\}}\left\{\norm{\mQ_m^{(l)}}, \norm{\mK_m^{(l)}}\right\}+\sum_{m=1}^M\norm{\mV_m^{(l)}}+\norm{\mW_1^{(l)}}+\norm{\mW_2^{(l)}}\right\},
\end{equation}
where $\norm{\cdot}$ for matrices indicates the operator norm in this equation.

\end{description}

\section{Approximating UDA Algorithms by Transformers}

This section demonstrates that transformers in in-context learning have the ability to approximate existing UDA algorithms. Specifically, we show that transformers can approximate uLSIF-based IWL, an instance-based method, and DANN, a feature-based method.

\subsection{Setup and Notion}

We consider in-context domain adaptation, where a fixed Transformer model is given a tuple $(\gD_S, \gD_T, \vx_*)$, where $\vx_*\sim p_T^X(\vx)$, and predicts $y_*$ without updating model parameters.

We construct the input matrix using both source and target data.
Specifically, we suppose the input data, namely $(\vx_i^S, y_i^S)\in\gD_S$ and $\vx_i^T\in\gD_T$, are encoded into $\mH^{(1)}\in\R^{D\times (N+1)}$ as follows:
\begin{equation}
    \mH^{(1)} = \begin{bmatrix}
        \vx_1    &  \ldots  &   \vx_n &  \vx_{n+1}       & \ldots & \vx_{N}   & \vx_{N+1}    \\
        y_1      &  \ldots  &   y_n   &  0               & \ldots &  0         & 0      \\
        t_1      &  \ldots  &   t_n   &  t_{n+1}         & \ldots & t_{N}   & t_{N+1}      \\
        s_1      &  \ldots  &   s_n   &  s_{n+1}         & \ldots & s_{N}   & s_{N+1}      \\
        1        &  \ldots  &   1     &  1               & \ldots & 1          & 1   \\
        \vzero_{D-(d+4)}    &  \ldots  &   \vzero_{D-(d+4)}     &  \vzero_{D-(d+4)}          & \ldots & \vzero_{D-(d+4)}          & \vzero_{D-(d+4)}
    \end{bmatrix}, \label{def:input_H}
\end{equation}
where $\vx_i=\vx_i^S$ and $y_i=y_i^S$ for $1\leq i \leq n$, $\vx_i=\vx_{i-n}^T$ for $n+1\leq i\leq N$, $\vx_{N+1}=\vx_*$.
$t_i=1$ for $1\leq i \leq n$ otherwise $0$ is used to indicate which data point is from the source domain, and $s_i=1$ for $1\leq i \leq N$ otherwise $0$ is used to mark training data.
We further define an output of a transformer for the prediction corresponding to $\vx_{N+1} = \vx_*$.
For an input $\mH^{(1)}$ and its corresponding output matrix $\tf_\vtheta(\mH^{(1)})$, we write its $(2,N+1)$-th element of $\tf_\vtheta(\mH^{(1)})$ as $\tf^*_\vtheta(\mH^{(1)})$.

We define the concepts necessary for our theoretical results.
To approximate smooth functions, such as loss functions $\ell$ and $\gamma$, we use the following notion.
    
\begin{definition}[$(\varepsilon, R, M, C)$-approximability by sum of ReLUs, \cite{bai2023transformers}]
    For $\varepsilon>0$ and $R\geq 1$, a function $g: \R^k\to\R$ is \emph{$(\epsilon, R, M, C)$-approximabile by sum of ReLUs} if there exist a function
    \begin{equation}
        f(\vz)=\sum_{m=1}^M c_m \sigma(\va_m^\top\vz+b_m)\quad\text{with}~\sum_{m=1}^M\abs{c_m}\leq C, ~\max_{m\in\{1,\dots,M\}} \norm{\va}_1+b_m\leq 1,
    \end{equation}
    where $ \va_m\in\R^k, b_m\in\R, c_m\in\R,$, such that $\sup_{\vz\in[-R,R]^k}\abs{g(\vz)-f(\vz)}\leq \varepsilon$.
\end{definition}

All functions approximated in this paper are included in this class.
To use this, we suppose $\norm{\vx_i}\leq B_x$ and $\norm{y_i}\leq B_y$ for any $i\in\{1,\dots,N+1\}$, and $\norm{\vw}\leq B_w$ for a model weight $\vw$ in in-context learning.
Additionally, for the feature map of IWL, we suppose $\norm{\vphi}\leq 1$ so that $\norm{\vphi(\vx)}\leq B_x$.

\subsection{In-context IWL with uLSIF}\label{sec:icl_ulsif}

We show a transformer in the ICL scheme can approximate the uLSIF estimator-based IWL algorithm, which we refer to as IWL in the following.

In this section, we consider the following setup:
Fix any $B_w>0$, $B_\alpha>0$, $L_1, L_2>0$, $\eta_1$, and $ \eta_2$.
Given a loss function $\ell$ that is convex in the first argument, and $\nabla_1\ell$ is $(\epsilon, R, M, C)$-approximable by the sum of ReLUs with $R=\max(B_wB_x, B_y, 1)$.
We first state the main result of this section.
This theorem shows that the transformer has the performance to function as an algorithm almost equivalent to IWL for any input $\vx_*$.
Note that the query $\vx_*$ is encoded in the input matrix $\mH^{(1)}$ defined in \eqref{def:input_H}.

\begin{theorem}\label{thm:ulsif_main}
    Consider $\hat{f}^{\mathrm{IWL}}$ with $\vphi$ which is approximable by a sum of ReLUs.
    Fix $\varepsilon > 0$ arbitrarily and set $L_2$ as satisfying $0<\epsilon\leq B_w/2L_2$.
    Suppose that an input $(\gD_S, \gD_T, \vx_*)$ satisfies that $\sup_{\vw:\norm{\vw}_2\leq B_w}\lambda_{\max} (\nabla^2\hat{R}_T(\vw; \gD_S))\leq \eta_2/2$ and the minimizer $\vw^*\in\argmin\hat{R}_T(\vw; \gD_S)$ satisfies $\norm{\vw^*}\leq B_w/2$.
    Then, there exists a transformer $\tf_\vtheta$ with $L_1+L_2+1$ layers and $M$ heads which satisfies the following:
    \begin{align}
        \|\tf_\vtheta^*(\mH^{(1)}) - \hat{f}^{\mathrm{IWL}}(\vx_*)\| \leq \varepsilon.
    \end{align}
\end{theorem}

This result gives several implications. 
First, while transformers with ICL have been shown to approximate algorithms for i.i.d. data \cite{bai2023transformers}, this theorem extends it to show the capability of handling domain shifts. 
This extension is nontrivial since it requires showing that the Transformer can express the algorithm's ability of corrections for the domain shifts.
Second, this theorem shows that Transformers do not need to specify the feature map $\vphi(\cdot)$.
Specifically, whatever feature map $\vphi(\cdot)$ is used for uLSIF, the transformer can approximate the uLSIF based on it. 
In other words, pre-training a transformer induces a situation-specific feature map $\vphi(\cdot)$.

\subsubsection*{Proof Outline}\label{sec:icl_ulsif_proof}

In our proof, we construct three types of sub-transformers and approximate IWL-uLSIF by combining them.
Specifically, the first sub-transformer with one layer to construct a feature map $\vphi(\cdot)$, the second one calculates $\valpha$ with $L_1$ layers, and this one optimizes the model weight $\vw\in\R^J$ with $L_2$ layers.
For simplicity, we define a notation $\vz_i = [\vx_i, y_i, t_i, s_i, 1]$.
In the following, we describe each approximation step by step.
The full proof is presented in \cref{ap:sec:icl_ulsif_proofs}.

\begin{enumerate}[label=\textbf{Step \arabic*:},leftmargin=*,itemindent=25pt]
    \item \textbf{Feature map approximation.} 
First, we construct a transformer that approximates the feature map $\vphi(\cdot)$. 
This construction is trivial under the assumption that the feature map $\vphi(\cdot)$ is approximable by the sum of ReLUs.
For example, a self-attention block can construct a feature map with the RBF kernel.
\begin{lemma}
    There exists a transformer with one layer which maps $\vh_i^{(1)}=[\vz_i, \vzero_{J}, \vzero_{D-(d+4)}]$ to $\tilde{\vh}_i^{(1)}=[\vz_i, \hat{\vphi}(\vx_i), \vzero_{D-J-(d+4)}] $ for $ i=1,...,N+1$, such that $\|\vphi - \hat{\vphi}\|_{L^\infty} \leq \varepsilon$ holds.
\end{lemma}

\item \textbf{Density ratio parameter approximation.} 
Second, we construct a transformer that approximates $\hat{\valpha}$ that minimizes the loss $\hat{L}(\valpha)$ in \cref{eq:iwl_alpha_estimation}, in which a technical difficulty lies. 
For the approximation, we define a sequence of parameters $\{\valpha_\mathrm{GD}^{(l)}\}_{l =1,2,...}$ that converge to $\hat{\valpha}$, using an update equation by the gradient descent algorithm:
\begin{align}
    \valpha_\mathrm{GD}^{(l +1)} = \valpha_\mathrm{GD}^{(l)} - \eta_1 \nabla \hat{L}(\valpha_\mathrm{GD}^{(l)}),~ l = 1,2,..., \mbox{~and~}\valpha_\mathrm{GD}^{(1)} = \vzero_J.
\end{align}
Then, we develop a Transformer layer that exactly implements a single update step of this equation.
We give the following statement.
\begin{lemma}
    There exists a Transformer with $L_1$ layer which maps $\tilde{\vh}_i^{(1)}=[\vz_i, \vphi(\vx_i), \vzero_J, \vzero_{D-2J-(d+4)}]$ to $\tilde{\vh}_i^{(L_1 + 1)}=[\vz_i, \vphi(\vx_i), \tilde{\valpha}^{(L_1)},$ $ \vzero_{D-2J-(d+4)}]$ for $ i=1,...,N+1$, such that ${\valpha}_\mathrm{GD}^{(L_1)}=\tilde{\valpha}^{(L_1)}$.
\end{lemma}
We note that this approach with the gradient descent is more efficient in the sense of layer size than the direct minimization of $\hat{L}(\valpha)$ using an inverse matrix.

\item \textbf{Model parameter approximation.}
In this step, we employ a similar approach to develop a Transformer that approximates $\vw^*$, which is the minimizer of $\hat{R}_T(\vw)$ in \cref{eq:iwl_weight_estimation}. 
To the aim, we define a sequence of parameters $\{{\vw}_\mathrm{GD}^{(l)}\}_{l =1,2,...}$ that converge to $\vw^*$ by the gradient descent algorithm:
\begin{align}
    \vw_\mathrm{GD}^{(l +1)} = \vw_\mathrm{GD}^{(l)} - \eta_2 \nabla \hat{R}_T(\vw_\mathrm{GD}^{(l)}),~ l = 1,2,..., \text{~and~}\vw_\mathrm{GD}^{(1)} = \vzero_J.
\end{align}
Then, we develop a Transformer layer that approximates this algorithm as follows:
\begin{lemma}
    There exists a Transformer with $L_2$ layer which maps $\tilde{\vh}_i^{(L_1 + 1)}=[\vz_i, \vphi(\vx_i),  \tilde{\valpha}^{(L_1)}, \vzero_{D-2J-(d+4)}]$ to $\tilde{\vh}_i^{(L_1+L_2+1)}=[\vz_i,  {\vphi}(\vx_i),  \tilde{\valpha}^{(L_1)},\tilde{\vw}^{(L_2)}, 0]$ for $ i=1,...,N+1$, such that $\|{\vw}_\mathrm{GD}^{(L_2)} - \tilde{\vw}^{(L_2)}\|_{L^\infty} \leq \varepsilon L_2 \eta_2 B_x$ holds.
\end{lemma}
Combining these layers constitutes the transformer that achieves the approximation capability of Theorem \ref{thm:ulsif_main}. 
Importantly, this method requires as many layers as there are iterations, but this can be shortened by using regression coupling.
\end{enumerate}

\subsection{In-context DANN}\label{sec:icl_dann}
We show the existence of a Transformer that can approximate DANN, presented in Section \ref{sec:uda_preliminary}. %

At the beginning, we identify the structure of the DANN to be approximated.
First, we assume that a label classifier $f_L\circ f_F(\vx)$ is represented by a two-layer neural network model $\Lambda(\vx; \mU, \vw):=\sum_{k=1}^K w_k r(\vu_k^\top \vx)$, and the domain classifier $f_D\circ f_F(\vx)$ also follows a similar model $\Delta(\vx; \mU, \vv):=\sum_{k=1}^K v_k r(\vu_k^\top \vx)$, with parameters $\vw, \vv\in\R^K$ and $\mU=[\vu_1,\dots\vu_K]\in\R^{K\times d}$ and an activation function $r(\cdot)$.
Let $\vu=\mathrm{vec}(\mU)$ and $\vw, \vv, \vu$ correspond to $\vtheta_L, \vtheta_D, \vtheta_F$, respectively. %
This simplification using a two-layer neural network was included to simplify the discussion and is easy to generalize.

We further define a sequence of parameters $\{(\vu^{(l)}, \vv^{(l)}, \vw^{(l)})\}_{l=1}^L$ that can be obtained by $L$ updates by the DANN.
In preparation, we define a closed set $\gW=\{(\vw, \vv, \vu): \norm{\vw}\leq B_w, \norm{\vv}\leq B_v, \max_k\norm{\vu_k}\leq B_u\} \subset \R^{K \times K \times K d}$ with some $B_u, B_w, B_v > 0$ and also define a projection $\Pi_\gW(\cdot)$ onto $\gW$.
Let $(\vw^{(1)}, \vv^{(1)}, \vu^{(1)}) \in \gW$ be a tuple of initial values.
Then, the subsequent parameters are recursively defined by the following updates:
    \begin{align}
        {\vu}^{(l)} &= \Pi_\gW\left( {\vu}^{(l-1)}-\eta\nabla_\vu\{{L}({\vu}^{(l-1)}, {\vw}^{(l-1)})-\lambda{\Omega}({\vu}^{(l-1)}, {\vv}^{(l-1)})\} \right), \label{eq:dann1_2}\\
        {\vw}^{(l)} &= \Pi_\gW\left({\vw}^{(l-1)}-\eta\nabla_\vw{L}({\vu}^{(l-1)}, {\vw}^{(l-1)})\right), \label{eq:dann2_2} \\
        {\vv}^{(l)} &= \Pi_\gW\left({\vv}^{(l-1)}-\eta\lambda\nabla_\vv{\Omega}({\vu}^{(l-1)}, {\vv}^{(l-1)})\right). \label{eq:dann3_2}
    \end{align}
These updates are an analogy of the original DANN updates \cref{eq:dann1,eq:dann2,eq:dann3} to optimize the minimax loss function \cref{eq:dann_minimax} of DANN depending on $\gD_S$ and $\gD_T$.
Here, we introduce the projection to explain a more realistic setup, for example, \cite{shen2018wasserstein}.
Finally, the classifier obtained by $L$ updates of DANN is defined as
\begin{align}
\hat{f}^{\mathrm{DANN}}(\vx) := \Lambda(\vx; \vu^{(L)}, \vw^{(L)}). \label{def:dann_classifier}
\end{align}

We show the existence of a Transformer that successively approximates the parameter sequence by DANN defined above. 
The results are given in the following statement.
\begin{theorem}\label{thm:dann}
    Fix any $B_u, B_w, B_v>0$, $L>0$, $\eta>0$, and $\epsilon>0$.
    Suppose activation function $r$ is $C^4$-smooth. %
    Then, there exists a $2L$-layer Transformer $\tf_\vtheta$ with
    \begin{equation}
        \max_{l\in\{1,\dots,2L\}} M^{(l)}\leq\tilde{\gO}(\epsilon^{-2}), \quad
        \max_{l\in\{1,\dots,2L\}} D^{(l)}\leq\tilde{\gO}(\epsilon^{-2})+D_{\mathrm{MLP}}, \quad
        \norm{\vtheta}_\tf \leq \gO(1+\eta)+C_{\mathrm{MLP}},
    \end{equation}
    with some existing constants $D_{\mathrm{MLP}},C_{\mathrm{MLP}} > 0$,
    which satisfies the following: 
    for any input $(\gD_S, \gD_T, \vx_*)$, an $2l$-th layer of the transformer maps $\tilde{\vh}_i^{(2l-1)} :=[\vz_i, \tilde{\vu}^{(l-1)}, \tilde{\vw}^{(l-1)}, \tilde{\vv}^{(l-1)}]$ with any $(\tilde{\vw}^{(l-1)}, \tilde{\vv}^{(l-1)}, \tilde{\vu}^{(l-1)} ) \in \gW$ to $\tilde{\vh}_i^{(2l)} :=[\vz_i, \tilde{\vu}^{(l)}, \tilde{\vw}^{(l)}, \tilde{\vv}^{(l)}]$ for each $i\in\{1,\dots,N+1\}$ such that 
    \begin{align}
        \tilde{\vu}^{(l)} &= \Pi_\gW\left( \tilde{\vu}^{(l-1)}-\eta\nabla_\vu\{\tilde{L}(\tilde{\vu}^{(l-1)}, \tilde{\vw}^{(l-1)})-\lambda\tilde{\Omega}(\tilde{\vu}^{(l-1)}, \tilde{\vv}^{(l-1)})\} +\vepsilon^{(l-1)}_u \right),\\
        \tilde{\vw}^{(l)} &= \Pi_\gW\left(\tilde{\vw}^{(l-1)}-\eta\nabla_\vw\tilde{L}(\tilde{\vu}^{(l-1)}, \tilde{\vw}^{(l-1)})+\vepsilon^{(l-1)}_w\right), \\
        \tilde{\vv}^{(l)} &= \Pi_\gW\left(\tilde{\vv}^{(l-1)}-\eta\lambda\nabla_\vv\tilde{\Omega}(\tilde{\vu}^{(l-1)}, \tilde{\vv}^{(l-1)})+\vepsilon^{(l-1)}_v\right),
    \end{align}
    where $\vepsilon^{(l-1)}_u, \vepsilon^{(l-1)}_w, \vepsilon^{(l-1)}_v \in \R^k$ are some vectors satisfying $ \max\{\|{\vepsilon^{(l-1)}_u}\|, \|{\vepsilon^{(l-1)}_w}\|, \|{\vepsilon^{(l-1)}_v}\|\}\leq \epsilon$.
\end{theorem}

The results show that the existing Transformer approximates the original DANN updates \cref{eq:dann1_2,eq:dann2_2,eq:dann3_2}, at each step. 
The terms $\vepsilon^{(l-1)}_u, \vepsilon^{(l-1)}_w, \vepsilon^{(l-1)}_v$ expresses approximation errors, whose norm are no more than the fixed $\epsilon$. 
Consequently, the final output of the transformer approximates the output of the original DANN after $L$-iterations. The results are given as follows:
\begin{corollary}
    Consider the setup as in Theorem \ref{thm:dann}.
    Then, for any input $(\gD_S, \gD_T, \vx_*)$, the transformer $\mathrm{TF}_{\vtheta}$ in \cref{thm:dann} outputs a corresponding tuple $(\tilde{\vw}^{(L)}, \tilde{\vv}^{(L)}, \tilde{\vu}^{(L)} ) \in \gW$ which satisfies
    \begin{align}
        \|\tf^*_{\vtheta}(\mH^1) - \hat{f}^{\mathrm{DANN}}(\vx_*) \| \leq L \varepsilon.
    \end{align}
\end{corollary}

\subsubsection*{Proof Outline}\label{sec:icl_dann_proof}

Here, we focus on the sub-problem that a pair of Transformer layers approximates a single step of update of DANN consisting of two-layer neural networks.
Specifically, the attention block at the first layer approximates the forward passes of networks, the succeeding MLP block computes the partial derivatives of the loss function, then the next attention block implements the minimax optimization step, and the last MLP block projects the weights.

The specific procedure is presented in the following steps.
For brebity, we write $\vgamma := (\vu, \vw, \vv)$.

\begin{enumerate}[label=\textbf{Step \arabic*:},leftmargin=*,itemindent=25pt]
    \item \textbf{Forward pass approximation.} The first attention block performs the forward passes, that is, $\vx_i\mapsto(\Lambda(\vx_i), \Delta(\vx_i))$, by approximating the activation function $r$ by sum-of-ReLUs.
    \begin{lemma}
        For each $\varepsilon > 0$, there exists an attention block which maps $\vh_i=[\vz_i, \vgamma, \vzero_4]$ to $\vh'_i=[\vz_i, \vgamma, \tilde{\Lambda}(\vx_i), \tilde{\Delta}(\vx_i), \vzero_2]$ for $i=1,\dots,N$ such that $\|{\tilde{\Lambda}-\Lambda}\|_{L^\infty}\leq \varepsilon$ and $\|{\tilde{\Delta}-\Delta}\|_{L^\infty}\leq \varepsilon$ hold.
    \end{lemma}
    \item \textbf{Loss derivative approximation.} The next MLP block obtains loss derivatives, i.e., $(\tilde{\Lambda}(\vx_i), \tilde{\Delta}(\vx_i), y_i, t_i, s_i)\mapsto(\1(i\leq n)\partial_1\gamma(\tilde{\Lambda}(\vx_i), y_i), \1(i\leq N)\partial_1\gamma(\tilde{\Delta}(\vx_i), y_i))$.
    We also use sum-of-ReLUs to approximate $\partial_1\gamma$.
    \begin{lemma}
        For any $\varepsilon > 0$, there exists an MLP block which maps $\vh'_i=[\vz_i, \vgamma, \tilde{\Lambda}(\vx_i), \tilde{\Delta}(\vx_i), \vzero_2]$ to $\vh''_i=[\vz_i, \vgamma, \tilde{\Lambda}(\vx_i), \tilde{\Delta}(\vx_i), g_{\Lambda,i}, g_{\Delta,i}]$, where $g_{\Lambda,i}$ and $ g_{\Delta,i}$ are scalars satisfying $|{g_{\Lambda,i}-\1(i\leq n)\partial_1\gamma(\tilde{\Lambda}(\vx_i), y_i)}|\leq \varepsilon$ and $|{g_{\Delta,i}-\1(i\leq N)\partial_1\gamma(\tilde{\Delta}(\vx_i), y_i)}|\leq \varepsilon$.
    \end{lemma}
    \item \textbf{Gradient descent approximation} The attention block at the second layer approximates the optimization step: $(\vu, \vw, \vv)\mapsto (\vu-\eta\nabla_\vu(L(\vu, \vw)-\lambda\Omega(\vu, \vv)), \vw-\eta\nabla_\vw L(\vu, \vw), \vv-\eta\lambda\nabla_\vv \Omega(\vu, \vv))$, which can be obtained by approximating $(s, t)\mapsto s\cdot r'(t)$.
    \begin{lemma}
        For each $\varepsilon > 0$, there exists an attention block that maps $\vh''_i=[\vz_i, \vgamma, \tilde{\Lambda}(\vx_i), \tilde{\Delta}(\vx_i), g_{\Lambda,i}, g_{\Delta,i}]$ to $\vh'''_i=[\vz_i, \vu-\eta\vg_u, \vw-\eta\vg_w, \vv-\eta\vg_v, \tilde{\Lambda}(\vx_i), \tilde{\Delta}(\vx_i), g_{\Lambda,i}, g_{\Delta,i}]$, where $\|{\vg_u-\nabla_\vu(\tilde{L}(\vu, \vw)-\lambda\tilde{\Omega}(\vu, \vv))}\|_2, \|{\vg_w-\nabla_\vw\tilde{L}(\vu, \vw)}\|_2, \|{\vg_v-\lambda\nabla_\vv\tilde{\Omega}(\vu, \vv)}\|_2\leq \varepsilon$.
    \end{lemma}
    \item \textbf{Projection approximation.} The last MLP block projects $\vu, \vw, \vv$ onto $\gW$ appropriately by the assumption. 
\end{enumerate}

By stacking such pairs of such layers, a $2L$-layer Transformer can implement $L$ steps of DANN in context. 
The detailed derivation is presented in \cref{ap:sec:icl_dann_proofs}.

\section{Automatic Algorithm Selection by In-Context UDA}

We demonstrate that Transformers can automatically select UDA algorithms based on data in context. 
Specifically, we consider the selection of methods determined by whether the supports of the source distribution $p_S$ and the target distribution $p_T$ sufficiently overlap.
In cases where the support of the target distribution $p_T$ overlaps with that of the source distribution $p_S$, the IWL algorithm, which operates using the density ratio $q(\vx)=p_T(\vx)/p_S(\vx)$, is employed. 
In other words, IWL should be employed when we have
\begin{align}
    p_S(\vx) > 0 \mbox{~holds~for~all~} \vx \mbox{~such~that~} p_T(\vx) > 0. \label{eq:condition_density_ratio}
\end{align}
Conversely, when there is no such overlap, i.e., $p_S(\vx) = 0$ for some $\vx$ with $p_T(\vx) > 0$, one should select DANN, which does not rely on the density ratio.

We show that Transformers are capable of realizing the above design automatically.
\begin{align}
    \hat{f}^{\mathrm{ICUDA}}(\vx) := 
    \begin{cases}
        \hat{f}^{\mathrm{IWL}}(\vx_*) & \text{~if~} p_S(\vx) > 0, \forall \vx \text{~such~that~} p_T(\vx) > 0, \\
        \hat{f}^{\mathrm{DANN}}(\vx_*) & \text{~otherwise},
    \end{cases} 
\end{align}

The statement is as follows:
\begin{theorem}\label{thm:algorithm_selection}
Fix any $\epsilon > 0$.
Suppose that $n' \geq (1/\epsilon)^{3}\log n'$ holds and $p_T(\cdot)$ is Lipschitz continuous.
Then, there exists a Transformer $\tf_\vtheta$ with three layers and $M$ heads satisfies the following with probability at least $1-1/n' - O(\epsilon)$:
for any input $(\gD_S, \gD_T, \vx_*)$, a transformer $\mathrm{TF}_\vtheta$ satisfies 
\begin{align}
    \|\mathrm{TF}^*_\vtheta(\mH^1)  - \hat{f}^{\mathrm{ICUDA}}(\vx_*)\| \leq \epsilon.
\end{align}
\end{theorem}
The results show that the Transformer can automatically check the condition \eqref{eq:condition_density_ratio} and use the appropriate algorithm without the user having to make a choice. 
Note that the density ratio condition \eqref{eq:condition_density_ratio} is one of the options that are automatically learned, and in practice, Transformers can learn more complicated conditions. 
Importantly, in this case, cross-validation cannot be used to make a selection since true labels of the target domain data cannot be observed.

\subsubsection*{Proof Outline}
The key idea is to first implement kernel density estimation of $p_S(\vx)$ for $\vx\in\gD_T$ and then select an algorithm with $\min_{\vx\in\gD_T}\hat{p}_S(\vx)>\delta$, where $\delta>0$ is a given threshold.
Here, we assume that each token is $\vh_i=[\vz_i, \vgamma, \tilde{f}^{\mathrm{IWL}}(\vx_i), \tilde{f}^{\mathrm{DANN}}(\vx_i), \vzero]$, where $\vgamma$ consists of weights of in-context IWL and DANN and is unused in this selection process.
Further details are discussed in \cref{ap:sec:proof_icuda}.

\begin{enumerate}[label=\textbf{Step \arabic*:},leftmargin=*,itemindent=25pt]
    \item \textbf{Approximation of density estimation.} The first step is to approximate the density function $p_S$ with kernel density estimation $\hat{p}_S(\cdot)=\frac{1}{n}\sum_{i=1}^n K(\cdot, \vx_i)$, where $K$ is a kernel, such as the RBF kernel, and evaluate it on each data point.
    \begin{lemma}
        For any $\epsilon>0$, there exists an attention block that maps $\vh_i$ to $\vh'_i=[\vz_i, \vgamma, \tilde{f}^{\mathrm{IWL}}(\vx_i), \tilde{f}^{\mathrm{DANN}}(\vx_i), p_i, \vzero]$, where $\abs{p_i-\bar{p}_S(\vx_i)}\leq \epsilon$.
    \end{lemma}
    \item \textbf{Algorithm selection approximation.} Then, the Transformer decides which algorithm is appropriate based on $\softmin_{\beta,i} p_i > \delta$, where $\softmin_{\beta, i} p_i=-\frac{1}{\beta}\log\sum(-\beta p_i)$ is a relaxation of $\min$ with an inverse temperature parameter $\beta>0$.
    \begin{lemma}
        For any $\epsilon>0$, there is an MLP block succeeded by a Transformer layer that maps $\vh'_{N+1}$ to $\vh''_{N+1}=[\vz_i, \vgamma, \tilde{f}^{\mathrm{IWL}}(\vx_*), \tilde{f}^{\mathrm{DANN}}(\vx_*), p_i, -\frac{1}{\beta}\log\sum_i \exp(-\beta p_i), \tilde{f}(\vx_*)]$, where $|{\tilde{f}-\hat{f}^{\mathrm{ICUDA}}}|\leq\epsilon$.
    \end{lemma}
\end{enumerate}

\section{Experiments}

We verify our theory by using the in-context domain adaptation abilities of Transformers using two synthetic problems.
The ICL domain adapter is compared with
\begin{enumerate*}
    \item IWL with uLSIF using the RBF kernel;
    \item DANN of a two-layer neural network with the ELU activation function; and
    \item the same neural network trained on the source domain only.
\end{enumerate*}
For the in-context learning, we used an eight-layer Transformer with eight heads and pre-trained it to minimize $\gamma(\tf_\vtheta(\vx'_*; \gD'_S, \gD'_T), y'_*)$, where $(\vx'_*, y'_*)\sim\gD'_T$, for randomly synthesized datasets $(\gD'_S, \gD'_T) (\neq (\gD_T, \gD_S))$ for $10^4$ iterations.
Each dataset consists of $n=n'=N/2$ data points. 
Note that test data used to report test accuracy are unseen during pre-training.
Further experimental details can be found in \cref{ap:sec:exp_settings} and the source is included in the supplementary material.

\Cref{fig:experiments} (Left) presents test accuracy on the two-moon 2D problem, where the target distribution is a rotation of the source one.
The decision boundaries are also presented, and we can observe that the Transformer learns a smoother boundary than others.
Additionally, \cref{fig:experiments} (Right) demonstrates the results of the colorized MNIST problem, where each dataset consists of images of two digits resized to $8\times 8$ pixels, and their background colors alter based on the domains, as presented at the rightmost.
On both problems, ICL consistently achieves much better performance than the baselines.
These results indicate that Transformer implements adaptive domain adaptation algorithms to given datasets.

\begin{figure}[t]
    \centering
    \includegraphics[width=\linewidth]{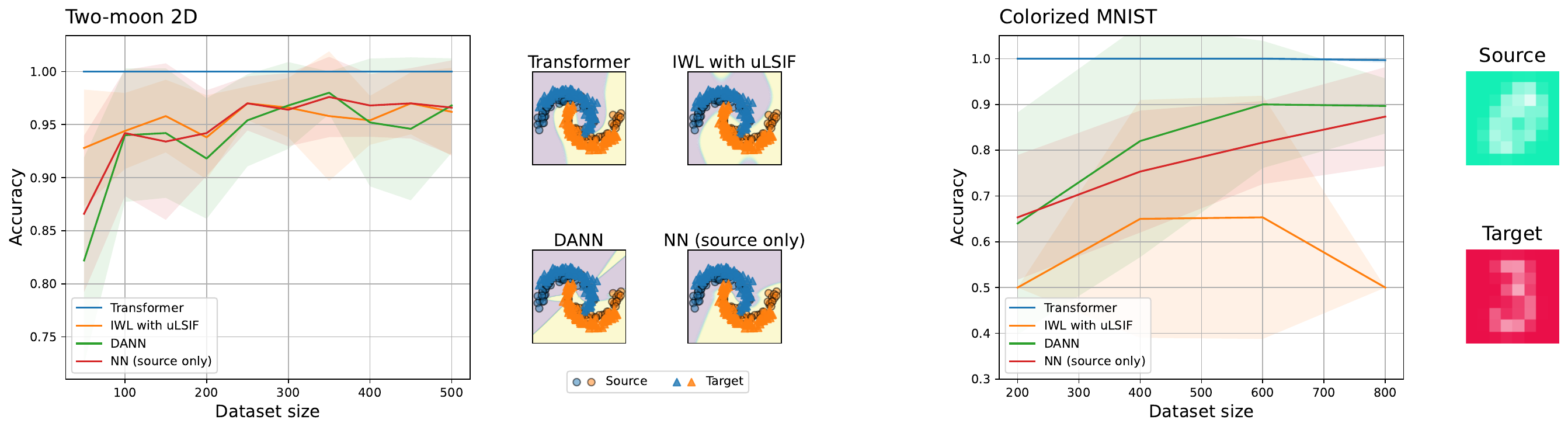}
    \caption{(Left) Test accuracy averaged over five runs of Transformer (ICL) and baseline models on the two-moon 2D problem.
Decision boundaries of the models are presented when $N=200$.
(Right) Test accuracy averaged over five runs on the colorized MNIST. 
    }
    \label{fig:experiments}
\end{figure}

\section{Conclusion and Discussion}

This paper proved that Transformers can approximate instance-based and feature-based domain adaptation algorithms and select an appropriate one for the given dataset in the in-context learning framework.
Technically, our results revealed that Transformers can approximately solve some types of linear equations and minimax problems.
Numerical experiments demonstrated that domain adaptation in in-context learning outperforms UDA methods tailored for specific assumptions.

\begin{description}[leftmargin=*,topsep=0pt]
    \item[Limitations] 
    In this work, we presented that Transformer models can approximate popular UDA methods, namely, importance-weighted learning with the uLSIF estimator and domain adversarial neural networks.
    However, the approximability of other domain adaptation techniques, such as parameter-based methods, is left unclear, and investigating it is an interesting direction.
    Furthermore, investigating whether appropriate domain adaptation algorithms can be automatically selected from a larger pool of candidates remains a task for future work.
\end{description}

\bibliography{main}

\begin{thebibliography}{28}
\providecommand{\natexlab}[1]{#1}
\providecommand{\url}[1]{\texttt{#1}}
\expandafter\ifx\csname urlstyle\endcsname\relax
  \providecommand{\doi}[1]{doi: #1}\else
  \providecommand{\doi}{doi: \begingroup \urlstyle{rm}\Url}\fi

\bibitem[Ahn et~al.(2024)Ahn, Cheng, Daneshmand, and Sra]{ahn2024transformers}
Ahn, K., Cheng, X., Daneshmand, H., and Sra, S.
\newblock Transformers learn to implement preconditioned gradient descent for in-context learning.
\newblock \emph{Advances in Neural Information Processing Systems}, 36, 2024.

\bibitem[Aky{\"u}rek et~al.(2022)Aky{\"u}rek, Schuurmans, Andreas, Ma, and Zhou]{akyurek2022learning}
Aky{\"u}rek, E., Schuurmans, D., Andreas, J., Ma, T., and Zhou, D.
\newblock What learning algorithm is in-context learning? investigations with linear models.
\newblock \emph{arXiv preprint arXiv:2211.15661}, 2022.

\bibitem[Bai et~al.(2023)Bai, Chen, Wang, Xiong, and Mei]{bai2023transformers}
Bai, Y., Chen, F., Wang, H., Xiong, C., and Mei, S.
\newblock Transformers as statisticians: Provable in-context learning with in-context algorithm selection.
\newblock In \emph{Thirty-seventh Conference on Neural Information Processing Systems}, 2023.
\newblock URL \url{https://openreview.net/forum?id=liMSqUuVg9}.

\bibitem[Ben-David et~al.(2010)Ben-David, Blitzer, Crammer, Kulesza, Pereira, and Vaughan]{ben2010theory}
Ben-David, S., Blitzer, J., Crammer, K., Kulesza, A., Pereira, F., and Vaughan, J.~W.
\newblock A theory of learning from different domains.
\newblock \emph{Machine learning}, 79:\penalty0 151--175, 2010.

\bibitem[Dai et~al.(2007)Dai, Yang, Xue, and Yu]{dai2007boosting}
Dai, W., Yang, Q., Xue, G.-R., and Yu, Y.
\newblock Boosting for transfer learning.
\newblock In \emph{Proceedings of the 24th international conference on Machine learning}, pp.\  193--200, 2007.

\bibitem[Daum{\'e}~III(2009)]{daume2009frustratingly}
Daum{\'e}~III, H.
\newblock Frustratingly easy domain adaptation.
\newblock \emph{arXiv preprint arXiv:0907.1815}, 2009.

\bibitem[Ganin et~al.(2016)Ganin, Ustinova, Ajakan, Germain, Larochelle, Laviolette, March, and Lempitsky]{ganin2016domain}
Ganin, Y., Ustinova, E., Ajakan, H., Germain, P., Larochelle, H., Laviolette, F., March, M., and Lempitsky, V.
\newblock Domain-adversarial training of neural networks.
\newblock \emph{Journal of machine learning research}, 17\penalty0 (59):\penalty0 1--35, 2016.

\bibitem[Garg et~al.(2022)Garg, Tsipras, Liang, and Valiant]{garg2022can}
Garg, S., Tsipras, D., Liang, P.~S., and Valiant, G.
\newblock What can transformers learn in-context? a case study of simple function classes.
\newblock \emph{Advances in Neural Information Processing Systems}, 35:\penalty0 30583--30598, 2022.

\bibitem[Guan \& Liu(2021)Guan and Liu]{guan2021domain}
Guan, H. and Liu, M.
\newblock Domain adaptation for medical image analysis: a survey.
\newblock \emph{IEEE Transactions on Biomedical Engineering}, 69\penalty0 (3):\penalty0 1173--1185, 2021.

\bibitem[Jiang(2017)]{jiang2017uniform}
Jiang, H.
\newblock Uniform convergence rates for kernel density estimation.
\newblock In \emph{International Conference on Machine Learning}, pp.\  1694--1703. PMLR, 2017.

\bibitem[Kanamori et~al.(2009)Kanamori, Hido, and Sugiyama]{kanamori2009least}
Kanamori, T., Hido, S., and Sugiyama, M.
\newblock A least-squares approach to direct importance estimation.
\newblock \emph{The Journal of Machine Learning Research}, 10:\penalty0 1391--1445, 2009.

\bibitem[Kimura \& Hino(2024)Kimura and Hino]{kimura2024a}
Kimura, M. and Hino, H.
\newblock A short survey on importance weighting for machine learning.
\newblock \emph{Transactions on Machine Learning Research}, 2024.
\newblock ISSN 2835-8856.
\newblock URL \url{https://openreview.net/forum?id=IhXM3g2gxg}.
\newblock Survey Certification.

\bibitem[Kingma \& Ba(2015)Kingma and Ba]{adam}
Kingma, D.~P. and Ba, J.~L.
\newblock {Adam: a Method for Stochastic Optimization}.
\newblock In \emph{ICLR}, 2015.

\bibitem[Li et~al.(2023)Li, Ildiz, Papailiopoulos, and Oymak]{li2023transformers}
Li, Y., Ildiz, M.~E., Papailiopoulos, D., and Oymak, S.
\newblock Transformers as algorithms: Generalization and implicit model selection in in-context learning.
\newblock \emph{arXiv preprint arXiv:2301.07067}, 2023.

\bibitem[Lin et~al.(2023)Lin, Bai, and Mei]{lin2023transformers}
Lin, L., Bai, Y., and Mei, S.
\newblock Transformers as decision makers: Provable in-context reinforcement learning via supervised pretraining.
\newblock In \emph{NeurIPS 2023 Foundation Models for Decision Making Workshop}, 2023.

\bibitem[Loshchilov \& Hutter(2019)Loshchilov and Hutter]{adamw}
Loshchilov, I. and Hutter, F.
\newblock Decoupled weight decay regularization.
\newblock In \emph{ICLR}, 2019.
\newblock URL \url{https://openreview.net/forum?id=Bkg6RiCqY7}.

\bibitem[Paszke et~al.(2019)Paszke, Gross, Massa, Lerer, Bradbury, Chanan, Killeen, Lin, Gimelshein, Antiga, Desmaison, Kopf, Yang, DeVito, Raison, Tejani, Chilamkurthy, Steiner, Fang, Bai, and Chintala]{pytorch}
Paszke, A., Gross, S., Massa, F., Lerer, A., Bradbury, J., Chanan, G., Killeen, T., Lin, Z., Gimelshein, N., Antiga, L., Desmaison, A., Kopf, A., Yang, E., DeVito, Z., Raison, M., Tejani, A., Chilamkurthy, S., Steiner, B., Fang, L., Bai, J., and Chintala, S.
\newblock Pytorch: An imperative style, high-performance deep learning library.
\newblock In \emph{NeurIPS}, 2019.

\bibitem[Ravent{\'o}s et~al.(2024)Ravent{\'o}s, Paul, Chen, and Ganguli]{raventos2024pretraining}
Ravent{\'o}s, A., Paul, M., Chen, F., and Ganguli, S.
\newblock Pretraining task diversity and the emergence of non-bayesian in-context learning for regression.
\newblock \emph{Advances in Neural Information Processing Systems}, 36, 2024.

\bibitem[Requeima et~al.(2019)Requeima, Gordon, Bronskill, Nowozin, and Turner]{requeima2019fast}
Requeima, J., Gordon, J., Bronskill, J., Nowozin, S., and Turner, R.~E.
\newblock Fast and flexible multi-task classification using conditional neural adaptive processes.
\newblock \emph{Advances in Neural Information Processing Systems}, 32, 2019.

\bibitem[Shen et~al.(2018)Shen, Qu, Zhang, and Yu]{shen2018wasserstein}
Shen, J., Qu, Y., Zhang, W., and Yu, Y.
\newblock Wasserstein distance guided representation learning for domain adaptation.
\newblock In \emph{Proceedings of the AAAI conference on artificial intelligence}, 2018.

\bibitem[Shimodaira(2000)]{shimodaira2000improving}
Shimodaira, H.
\newblock Improving predictive inference under covariate shift by weighting the log-likelihood function.
\newblock \emph{Journal of statistical planning and inference}, 90\penalty0 (2):\penalty0 227--244, 2000.

\bibitem[Sugiyama et~al.(2007)Sugiyama, Krauledat, and M{\"u}ller]{sugiyama2007covariate}
Sugiyama, M., Krauledat, M., and M{\"u}ller, K.-R.
\newblock Covariate shift adaptation by importance weighted cross validation.
\newblock \emph{Journal of Machine Learning Research}, 8\penalty0 (5), 2007.

\bibitem[Sugiyama et~al.(2012)Sugiyama, Suzuki, and Kanamori]{sugiyama2012density}
Sugiyama, M., Suzuki, T., and Kanamori, T.
\newblock \emph{Density ratio estimation in machine learning}.
\newblock Cambridge University Press, 2012.

\bibitem[Sun et~al.(2019)Sun, Liu, Chua, and Schiele]{SunLCS2019MTL}
Sun, Q., Liu, Y., Chua, T., and Schiele, B.
\newblock Meta-transfer learning for few-shot learning.
\newblock In \emph{{IEEE} Conference on Computer Vision and Pattern Recognition, {CVPR} 2019, Long Beach, CA, USA, June 16-20, 2019}, pp.\  403--412. Computer Vision Foundation / {IEEE}, 2019.

\bibitem[Von~Oswald et~al.(2023)Von~Oswald, Niklasson, Randazzo, Sacramento, Mordvintsev, Zhmoginov, and Vladymyrov]{von2023transformers}
Von~Oswald, J., Niklasson, E., Randazzo, E., Sacramento, J., Mordvintsev, A., Zhmoginov, A., and Vladymyrov, M.
\newblock Transformers learn in-context by gradient descent.
\newblock In \emph{International Conference on Machine Learning}, pp.\  35151--35174. PMLR, 2023.

\bibitem[Xie et~al.(2021)Xie, Raghunathan, Liang, and Ma]{xie2021explanation}
Xie, S.~M., Raghunathan, A., Liang, P., and Ma, T.
\newblock An explanation of in-context learning as implicit bayesian inference.
\newblock In \emph{International Conference on Learning Representations}, 2021.

\bibitem[Ying et~al.(2018)Ying, Zhang, Huang, and Yang]{ying2018transfer}
Ying, W., Zhang, Y., Huang, J., and Yang, Q.
\newblock Transfer learning via learning to transfer.
\newblock In \emph{International conference on machine learning}, pp.\  5085--5094. PMLR, 2018.

\bibitem[Zhang et~al.(2023)Zhang, Frei, and Bartlett]{zhang2023trained}
Zhang, R., Frei, S., and Bartlett, P.
\newblock Trained transformers learn linear models in-context.
\newblock In \emph{R0-FoMo: Robustness of Few-shot and Zero-shot Learning in Large Foundation Models}, 2023.

\end{thebibliography}
\bibliographystyle{icml2024}

\appendix

\section{Proofs}\label{ap:sec:proofs}

\subsection{In-context IWL with uLSIF}\label{ap:sec:icl_ulsif_proofs}

We prove \cref{thm:ulsif_main} in three steps: the first layer applies some feature map to the input data, the next $L_1$ layers estimate $\valpha$, and then the last $L_2$ layers compute $\vw$ with importance-weighted learning.

\begin{proof}[Proof of \cref{thm:ulsif_main}]

    \begin{enumerate}[label=\textbf{Step \arabic*:},leftmargin=*,itemindent=30pt]
        \item \textbf{(The first layer)} It is trivial that the feature map $\vphi(\cdot)$ is approximable by the sum of ReLUs.
        As a typical example, \cref{ap:sec:proof_icuda} presents the approximation of the RBF kernel.
        \item \textbf{(The $2$nd to $L_1$th layers)} As $\mPsi=\frac{1}{n}\sum_{i=1}^n\vphi(\vx_i)\vphi(\vx_i)^\top\in\R^{J\times J}$, its vector product with $\valpha\in\R^J$ is $\mPsi\valpha=\frac{1}{n}\sum_{i=1}^n\vphi(\vx_i)\vphi(\vx_i)^\top\valpha=\frac{1}{n}\sum_{i=1}^n(\vphi(\vx_i)^\top\valpha)\vphi(\vx_i)$.
        Similarly, $\vpsi=\frac{1}{n'}\sum_{i=n+1}^{n+n'}\vphi(\vx_i)$.
        Thus, $\eta_1\nabla_\valpha\hat{L}(\valpha)=\eta_1(\mPsi\valpha-\vpsi+\lambda\valpha)$ can be constructed with matrices $\mQ_m^{(l)}, \mK_m^{(l)}, \mV_m^{(l)}$ for $m=1,2,3,4$ as follows:
        \begin{align}
            \mQ_1^{(l)}\vh_i &=\begin{bmatrix}
                \valpha^{(l)} \\
                -R' \\
                \vzero_{D-J-1}
            \end{bmatrix}, \quad
            \mK_1^{(l)}\vh_j &=\begin{bmatrix}
                \vphi(\vx_j) \\
                1-t_j \\
                \vzero_{D-J-1}
            \end{bmatrix}, \quad
            \mV_1^{(l)}\vh_j &= -\frac{(N+1)\eta_1}{n}\begin{bmatrix}
                \vzero_{D-J-d} \\
                \vphi(\vx_j) \\
                \vzero_d
            \end{bmatrix} \\
            \mQ_2^{(l)}\vh_i &=\begin{bmatrix}
                -\valpha^{(l)} \\
                -R' \\
                \vzero_{D-J-1}
            \end{bmatrix}, \quad
            \mK_2^{(l)}\vh_j &=\begin{bmatrix}
                \vphi(\vx_j) \\
                1-t_j \\
                \vzero_{D-J-1}
            \end{bmatrix}, \quad
            \mV_2^{(l)}\vh_j &= -\frac{(N+1)\eta_1}{n}\begin{bmatrix}
                \vzero_{D-J-d} \\
                \vphi(\vx_j) \\
                \vzero_d
            \end{bmatrix} \\
            \mQ_3^{(l)}\vh_i &= \begin{bmatrix}
                1 \\
                \vzero_{D-1}
            \end{bmatrix}, \quad
            \mK_3^{(l)}\vh_j &= \begin{bmatrix}
                s_j-t_j \\
                \vzero_{D-1}
            \end{bmatrix}, \quad
            \mV_3^{(l)}\vh_j &= \frac{(N+1)\eta_1}{n'}\begin{bmatrix}
                \vzero_{D-J-d} \\
                \vphi(\vx_j) \\
                \vzero_{d}
            \end{bmatrix} \\
            \mQ_4^{(l)}\vh_i &= \begin{bmatrix}
                1 \\
                \vzero_{D-1}
            \end{bmatrix}, \quad
            \mK_4^{(l)}\vh_j &= \begin{bmatrix}
                s_j \\
                \vzero_{D-1}
            \end{bmatrix}, \quad
            \mV_4^{(l)}\vh_j &= -\frac{(N+1)\lambda\eta_1}{N}\begin{bmatrix}
                \vzero_{D-J-d} \\
                \valpha^{(l)} \\
                \vzero_{d}
            \end{bmatrix},
        \end{align}
        where $R'=\max(B_x B_\alpha, 1)$. These matrices exist, and we obtain
        \begin{align}
            \tilde{\vh}_i^{(l)}
            =&\vh_i^{(l)} + \frac{1}{N+1}\sum_{j=1}^{N+1}\sum_{m=1}^4\sigma(\langle\mQ_m^{(l)}\vh_i,\mK_m^{(l)}\vh_j\rangle)\mV_m^{(l)}\vh_j \\
            =&\vh_i^{(l)} + \sum_{j=1}^{N+1}-\left\{\{\sigma(\vphi(\vx_j)^\top\valpha^{(l)}-R'(1-t_j))-\sigma(-\vphi(\vx_j)^\top\valpha^{(l)}-R'(1-t_j))\}[\vzero, \frac{\eta_1}{n}\vphi(\vx_j), \vzero]\right.\\
            &\left.+\sigma(s_j-t_j)[\vzero, \frac{\eta_1}{n'}\vphi(\vx_j), \vzero] + [\vzero, \frac{\eta_1\lambda}{N}\valpha^{(l)},  \vzero] \right\} \\
            =&\vh_i^{(l)}+(-\eta_1)\left\{\frac{1}{n}\sum_{j=1}^{n}\vphi(\vx_j)^\top\valpha^{(l)}[\vzero, \vphi(\vx_j), \vzero]-\sum_{j=n+1}^{n+n'}[\vzero, \vphi(\vx_j), \vzero] +[\vzero, \lambda\valpha^{(l)}, \vzero]\right\} \\
            =& [\vx_i, \dots, \underline{\valpha^{(l)}}, \vzero_J] + [\vzero_{D-2J}, \underline{-\eta_1(\mPsi\valpha^{(l)}-\vpsi+\lambda\valpha^{(l)})}, \vzero_J].
        \end{align}
        We used $z=\sigma(z)-\sigma(-z)$ for $z\in\R$, $\sigma(\vphi(\vx_j)^\top\valpha-R'(1-t_j))=\sigma(\vphi(\vx_j)^\top\valpha)\cdot \1(j\leq n)$, $\sigma(s_j)=\1(j\leq N)$, and $\sigma(s_j-t_j)=\1(n\leq j\leq N)$.
        The weight matrices have norm bounds of
        \begin{equation}
            \max_m\norm{\mQ_m^{(l)}}\leq 1+R, \quad \max_m\norm{\mK_m^{(l)}}\leq 1, \quad \sum_m\norm{\mV_m^{(l)}}\leq (\frac{2(N+1)}{n}+\frac{(N+1)}{n'}+\frac{\lambda}{N+1})\eta_1.
        \end{equation}

        \item \textbf{(The $L_1+1$th to $L_1+L_2+1$th layers)} As $(s, t, u)\mapsto u\partial_1\ell(s, t)$ is $(\epsilon, R, M, C)$-approximable by sum of ReLUs with $R=\max\{B_xB_w, B_y, 1\}$, there exists a function $f:[-R, R]^3\to\R$ of a form
        \begin{equation}
            f(s, t, u)=\sum_{m=1}^M c_m\sigma(a_m s+b_m t+d_m u + e_m)\quad\text{with}\quad\sum_{m=1}^M\abs{c_m}\leq C, \abs{a_m}+\abs{b_m}+\abs{d_m}+\abs{e_m}\leq 1~(\forall m),
        \end{equation}
        such that $\sup_{(s,t,u)\in[-R,R]^3}\abs{f(s,t,u)-u\partial_1\ell(s,t)}\leq \epsilon$.
        
        Then, there exists matrices $\mQ_m^{(l)}, \mK_m^{(l)}, \mV_m^{(l)}$ for $m\in\{1,\dots,M\}$ such that
        \begin{equation}
            \mQ_m^{(l)}\vh_i = \begin{bmatrix}
                a_m \hat\vw^{(l)} \\
                b_m \\
                d_m \valpha^{(L_1)} \\
                e_m \\
                -2
            \end{bmatrix},\quad
            \mK_m^{(l)}\vh_j=\begin{bmatrix}
                \vphi(\vx_j) \\
                y_j \\
                \vphi(\vx_j) \\
                1 \\
                R(1-t_j)
            \end{bmatrix},
            \mV_m^{(l)}\vh_j=-\frac{(N+1)c_m\eta_2}{n}\begin{bmatrix}
                \vzero_{D-d} \\
                \vphi(\vx_j)
            \end{bmatrix},
        \end{equation}
        with norm bounds of
        \begin{equation}
            \max_m\norm{\mQ_m^{(l)}}\leq 3, \quad \max_m\norm{\mK_m^{(l)}}\leq 2+R,\quad \sum_m\norm{\mV_m^{(l)}}\leq \frac{N+1}{n}C\eta_2.
        \end{equation}
        With these matrices, we get
        \begin{align*}
            \tilde{\vh}_i^{(l)} 
            &= \vh_i^{(l)} + \frac{1}{N+1}\sum_{j=1}^{N+1}\sum_{m=1}^M\sigma(\langle\mQ_m^{(l)}\vh_i,\mK_m^{(l)}\vh_j\rangle)\mV_m^{(l)}\vh_j  \\
            &= \vh_i^{(l)} + \sum_{j=1}^{n}\sum_{m=1}^M-\sigma(a_m\hat\vw^{(l)\top}\vphi(\vx_j)+b_m y_j+d_m\valpha^{(L_1)\top}\vphi(\vx_j)+e_m)\frac{c_m\eta_2}{n}[\vzero, \vphi(\vx_j)] \\
            &= \vh_i^{(l)} + \frac{\eta_2}{n}\sum_{j=1}^n -f(\hat\vw^{(l)\top}\vphi(\vx_j), y_j, \valpha^{(L_1)\top}\vphi(\vx_j))[\vzero, \vphi(\vx_j)] \\
            &= \vh_i^{(l)}+[0, -\frac{\eta_2}{n}\hat{q}_{\valpha}(\vx_j)\partial_1\ell(\hat\vw^{(l)\top}\vphi(\vx_j), y_j)\vphi(\vx_j)+\vepsilon]\\
            &= [\vx_i,\dots,\underline{\hat\vw^{(l)}}]+[\vzero_{D-d}, \underline{-\eta_2\nabla_\vw\hat{R}_T(\hat\vw^{(l)})+\vepsilon}],
        \end{align*}
        where the error $\vepsilon$ satisfies
        \begin{align}
            \norm{\vepsilon}_2
            &=\norm{-\frac{\eta_2}{n}\sum_{j=1}^n\abs{f(\hat\vw^{(l)\top}\vphi(\vx_j), y_j, \valpha^{(L_1)\top}\vphi(\vx_j))-\hat{q}_{\valpha^{(L_1)}}(\vx_j)\nabla_1\ell(\vw^{(l)\top}\phi(\vx_j), y_j)}\vphi(\vx_j)}_2 \\
            &\leq \frac{\eta_2}{n} \sum_{j=1}^n\abs{f(\hat\vw^{(l)\top}\vphi(\vx_j), y_j, \valpha^{(L_1)\top}\vphi(\vx_j))-\hat{q}_{\valpha^{(L_1)}}(\vx_j)\nabla_1\ell(\vw^{(l)\top}\phi(\vx_j), y_j)}\norm{\vphi(\vx_j)}_2 \\
            &\leq \frac{\eta_2}{n}n\epsilon B_x\cdot 1=\epsilon\eta_2 B_x,
        \end{align}
    \end{enumerate}
    since $\|\vphi\|\leq 1$ by the assumption.
    As a result, we get $\norm{\vtheta_\tf}\leq 1+R+\max\{(\frac{2(N+1)}{n}+\frac{(N+1)}{n'}+\frac{\lambda}{N+1})\eta_1, 1+\frac{N+1}{n}C\eta_2\}$.
\end{proof}

\subsection{In-context DANN}\label{ap:sec:icl_dann_proofs}

To prove \cref{thm:dann}, we show its single-step version.

\begin{proposition}\label{ap:prop:dann}
    Under the assumptions of \cref{thm:dann}, there exists a two-layer Transformer $\tf_\vtheta$ such that for any input $(\gD_S, \gD_T, \vx_*)$ and any $(\vw, \vv, \mU)\in \gW$, $\tf_\vtheta$ maps $\vh_i=[\vx_i, y_i, t_i, s_i, 1, \vu, \vw, \vv, \vzero]$ to $\vh'_i=[\vx_i, y_i, t_i, s_i, 1, \vu', \vw', \vv', \vzero]$, where
    \begin{align}
        \vu' &= \Pi_\gW\left\{\vu-\eta\nabla_\vu(\hat{L}(\vu, \vw)-\lambda\hat{\Omega}(\vu, \vv))+\vepsilon_u \right\}, \\
        \vw' &= \Pi_\gW\left\{\vw-\eta\nabla_\vw\hat{L}(\vu, \vw)+\vepsilon_w\right\}, \\
        \vv' &= \Pi_\gW\left\{\vv-\eta\lambda\nabla_\vv\hat{\Omega}(\vu, \vv)+\vepsilon_v\right\},
    \end{align}
    and
    \begin{equation}
        \norm{\vepsilon_u}, \norm{\vepsilon_w}, \norm{\vepsilon_v}\leq\eta \epsilon.
    \end{equation}
\end{proposition}

\begin{proof}[Proof of \cref{ap:prop:dann}]
Fix $\epsilon_r, \epsilon_p, \epsilon_\gamma$ that will be specified later. Our goal is to approximate the following values:
\begin{align}
    \nabla_{\vu_k} L(\vu, \vw) &= \frac{1}{n}\sum_{(\vx, y)\in\gD_S}\partial_1\gamma(\Lambda(\vx), y)\cdot \vw_k\cdot r'(\vu_k^\top\vx)\cdot\vx,\\
    \nabla_{\vw_k} L(\vu, \vw) &= \frac{1}{n}\sum_{(\vx, y)\in\gD_S}\partial_1\gamma(\Lambda(\vx), y)\cdot r(\vu_k^\top\vx),\\
    \nabla_{\vu_k} \Omega(\vu, \vv) &= -
    \frac{1}{n}\smashoperator{\sum_{(\vx, \cdot)\in\gD_S}}\partial_1\gamma(\Delta(\vx), 0)\cdot \vv_k\cdot r'(\vu_k^\top\vx)\cdot\vx 
    -   \frac{1}{n'}\smashoperator{\sum_{(\vx, \cdot)\in\gD_T}}\partial_1\gamma(\Delta(\vx), 1)\cdot \vv_k\cdot r'(\vu_k^\top\vx)\cdot\vx,\\
    \nabla_{\vv_k}\Omega(\vu, \vv)&=\frac{1}{n}\sum_{(\vx, \cdot)\in\gD_S}\partial_1\gamma(\Delta(\vx), 0)\cdot r(\vu_k^\top\vx) + \frac{1}{n'}\sum_{(\vx, \cdot)\in\gD_T}\partial_1\gamma(\Delta(\vx), 1)\cdot r(\vu_k^\top\vx),
\end{align}
Following \citep{bai2023transformers}, we show that a Transformer can approximate them in the following four steps.
\begin{enumerate}[label=\textbf{Step \arabic*:},leftmargin=*,itemindent=30pt]
    \item \textbf{(The attention block of the first layer)} As the activation function $r(t)$ is $(\epsilon_r, R_1, M_1, C_1)$-approximable for $R_1=\max(B_wB_x, B_vB_x)$ and $M_1\leq\tilde{O}(C_1^2\epsilon_r^{-2})$, where $C_1$ depends only on $R_1$ and the $C^2$-smoothness of $r$, there exists a function $\bar{r}: [-R_1, R_1]\to\R$ of a form
    \begin{equation}
        \bar{r}(t)=\sum_{m=1}^{M_1}c_m^1\sigma(a_m^1t+b_m^1)\quad\text{with}~\sum_{m=1}^{M_1}\abs{c_m^1}\leq C_1, ~\abs{a_m^1}+\abs{b_m^1}\leq 1~(\forall m),
    \end{equation}
    such that $\sum_{t\in[-R_1, R_1]\abs{r(t)-\bar{r}(t)}}\leq\epsilon_r$.
    Then, the forward pass of the networks $\Lambda$ and $\Delta$ can be approximated by a single self-attention layer.
    Namely, there exists matrices $\{\mQ_{k,m}, \mK_{k,m}, \mV_{k,m}\}_{k\in\{1,\dots,K\}}^{m\in\{1,\dots,M_1\}}$ such that
    \begin{equation}
        \mQ_{k,m}\vh_i=\begin{bmatrix}
            a_m^1\vx_i \\
            b_m^1 \\
            \vzero
        \end{bmatrix}, \quad
        \mK_{k,m}\vh_j=\begin{bmatrix}
            \vu_k \\
            1 \\
            0
        \end{bmatrix},\quad
        \mV_{k,m}\vh_j=c_m^1(\vw_k\cdot \ve_{D'+1}+\vv_k\cdot \ve_{D'+1}),
    \end{equation}
    where $D'=(K+3)d+4$ with norm bounds of
    \begin{equation}
        \max_{k,m}\norm{\mQ_{k,m}}\leq 1,\quad \max_{k,m}\norm{\mK_{k,m}}\leq 1,\quad \sum_{k,m}\norm{\mV_{k,m}}\leq C_1.
    \end{equation}
    Then,
    \begin{equation}
        \sum_{m,k}\sigma(\langle\mQ_{m,k}\vh_i,\mK_{m,k}\vh_j\rangle)\mV_{m,k}\vh_j = \sum_{k=1}^K\{\vw_k\bar{r}(\vu_k\vx_i)\cdot\ve_1+\vv\bar{r}(\vu_k^\top\vx_i)\cdot\ve_2\}.
    \end{equation}
    Denoting $\bar{\Lambda}(\vx)=\sum_{k=1}^K\vw_k\bar{r}(\vu_k^\top\vx)$ and $\bar{\Delta}(\vx)=\sum_{k=1}^K\vv_k\bar{r}(\vu_k^\top\vx)$, this first attention block maps $\vh_i$ to $\vh'_i=[\vx_i, y_i, t_i, s_i, 1, \vu, \vw, \vv, \bar{\Lambda}(\vx_i), \bar{\Delta}(\vx_i), \vzero_2]$.

    \item \textbf{(The MLP block of the first layer)} Because the function $(t,s)\mapsto\partial_1\gamma(t, s)$ is $(\epsilon_\gamma, R_2, M_2, C_2)$-approximable for $R_2=\max(1, B_xB_wB_u, B_xB_wB_v, B_y)$ and $M_2\leq\tilde{O}(C_2^2\epsilon_\gamma^{-2})$, where $C_2$ depends only on $R_2$ nad $C^3$-smootheness of $\partial_1\gamma$, there is a function $g: [-R_2, R_2]^2\to\R$ of a form
    \begin{equation}
        g(t,s) = \sum_{m=1}^{M_2} c_m^2(a_m^2t+b_m^2s+d_m^2) \quad\text{with}~\sum_{m=1}^{M_2}\abs{c_m^2}\leq C_2, \abs{a_m^2}+\abs{b_m^2}+\abs{d_m^2}\leq 1~(\forall m),
    \end{equation}
    such that $\sup_{(t,s)\in[-R_2, R_2]^2}\abs{g(t, s)-\partial_1\gamma(t, s)}\leq \epsilon_\gamma$.
    By picking matrices $W_1, W_2$ such that $W_1\in\R^{2M_2\times D}$ maps
    \begin{equation}
        \mW_1\vh'_i = \begin{cases}
            &[a_m^2\bar{\Lambda}(\vx_i)+b_m^2y_i+d_m^2-R_2(1-t_i)]_{m=1}^{M_2} \\
            &[a_m^2\bar{\Delta}(\vx_i)+b_m^2y_i+d_m^2-R_2(1-t_i)]_{m=M_2+1}^{2M_2}
        \end{cases}
    \end{equation}
    and $W_2\in\R^{D\times 2M_2}$ consists of elements of
    \begin{equation}
        (\mW_2)_{j,m}=\begin{cases}
            c_m^2\1(j=D'+3)\quad\text{if $1\leq m\leq M_2$} \\
            c_m^2\1(j=D'+4)\quad\text{otherwise}
        \end{cases}
    \end{equation}
    we have
    \begin{align}
        \mW_2\sigma(\mW_1\vh'_i)
        &=\sum_{m=1}^{M_2}\sigma(a_m^2\bar{\Lambda}(\vx_i)+b_m^2y_i+d_m^2-R_2(1-t_j))\cdot c_m^2\ve_{D'+3} \\
        &\quad+\sum_{m=M_2+1}^{2M_2}\sigma(a_m^2\bar{\Delta}(\vx_i)+b_m^2y_i+d_m^2-R_2(1-s_j))\cdot c_m^2\ve_{D'+4} \\   
        &=\1(i\leq n)\cdot g(\bar{\Lambda}(\vx_i), y_i) \cdot \ve_{D'+3} + \1(i\leq N)\cdot g(\bar{\Delta}(\vx_i), y_i)\cdot \ve_{D'+4},
    \end{align}
    where $\ve_k$ indicates a vector whose $i$th element is $\1(i=k)$. 
    Denoting $g_{\Lambda,i}=\1(i\leq n)\cdot g(\bar{\Lambda}(\vx_i), y_i)$ and $g_{\Delta,i}=\1(i\leq N)\cdot g(\bar{\Delta}(\vx_i), y_i)$, this MLP layer transforms $\vh'_i\mapsto \vh''_i=[\vx_i, y_i, t_i, s_i, 1, \vu, \vw, \vv, \bar{\Lambda}(\vx_i), \bar{\Delta}(\vx_i), g_{\Lambda,i}, g_{\Delta,i}]$.
    By definition of $g$,
    \begin{equation}
        \abs{g_{\Lambda,i}-\partial_1\gamma(\Lambda(\vx_i), y_i)}\leq \epsilon_\gamma+B_u L_\gamma\epsilon_r,
    \end{equation}
    where $L_\gamma=\max\abs{\partial_{11}\gamma(t,s)}$ is the smoothness of $\partial_1\gamma$.
    The same holds for $g_{\Delta,i}$.

    \item \textbf{(The self-attention block of the second layer)} The function $(t, s)\mapsto s\cdot r'(t)$ and $(t, s)\mapsto -s\cdot r'(t)$ are $(\epsilon_p, R_3, M_3, C_3)$-approximable for $R_3=\max\{B_xB_u,B_gB_u1\}$ and $M_3\leq\tilde{O}(C_3^2\epsilon_p^{-2})$, where $C_3$ depends only on $R_3$ and $C^3$-smoothness of $r'$.
    Thus, there exists a function
    \begin{equation}
        p(t, s) = \sum_{m=1}^{M-3}c_m^3\sigma(a_m^3t+b_m^3 s+d_m^3)\quad\text{with}~\sum{m=1}^{M_3}\abs{c_m^3}\leq C_3, ~\abs{a_m^3}+\abs{b_m^3}+\abs{d_m^3}\leq 1~(\forall m),
    \end{equation}
    such that $\sup_{(s,t)\in[-R_3,R_3]^2}\abs{p(s,t)-(\pm s)\cdot r'(t)}\leq\epsilon_p$.
    The gradients for gradient descent can be computed by using $p$ and $\bar{r}$.
    To this end, we use matrices $\{\mQ_{k,m,n}, \mK_{k,m,n}, \mV_{k,m,n}\}_{k=1,\dots,K}^{m=1,\dots,M_3}$ for $n=1,\dots,4$ such that
    \begin{align}
        \mQ_{k,m,1}\vh_i =& \begin{bmatrix}
            a_m^3 w_k \\
            b_m^3 \vu_k \\
            d_m^3 \\
            \vzero
        \end{bmatrix}, \quad
        \mK_{k,m,1}\vh_j =& \begin{bmatrix}
            g_{\Lambda,j} \\
            \vx_i \\
            1 \\
            \vzero
        \end{bmatrix}, \quad
        \mV_{k,m,1}\vh_j =& -\frac{(N+1)\eta c_m^3}{n}\begin{bmatrix}
            \vzero \\
            \vx_j \\
            \vzero
        \end{bmatrix} \\
        \mQ_{k,m,2}\vh_i =& \begin{bmatrix}
            a_m^1 \vu_k \\
            b_m^1 \\
            \vzero
        \end{bmatrix},\quad
        \mK_{k,m,2}\vh_j =&\begin{bmatrix}
            \vx_j \\
            1\\
            \vzero
        \end{bmatrix}, \quad
        \mV_{k,m,2}\vh_j =& -\frac{(N+1)\eta c_m^1}{n}\begin{bmatrix}
            \vzero \\
            g_{\Lambda,j} \\
            \vzero
        \end{bmatrix} \\
        \mQ_{k,m,3}\vh_i =& \begin{bmatrix}
            a_m^3v_k \\
            b_m^3\vu_k \\
            d_m^3 \\
            \vzero
        \end{bmatrix},\quad
        \mK_{k,m,3}\vh_j =& \begin{bmatrix}
            g_{\Delta,j}\\
            \vx_j \\
            1\\
            \vzero
        \end{bmatrix},\quad
        \mV_{k,m,3}\vh_j =& \frac{(N+1)\lambda\eta c_m^3}{N}\begin{bmatrix}
            \vzero \\
            \vx_j \\
            \vzero
        \end{bmatrix} \\
        \mQ_{k,m,4}\vh_i =& \begin{bmatrix}
            a_m^1\vu_k \\
            b_m^1 \\
            \vzero
        \end{bmatrix},\quad
        \mK_{k,m,4}\vh_j =& \begin{bmatrix}
            \vx_j \\
            1\\
            \vzero
        \end{bmatrix},\quad
        \mV_{k,m,4}\vh_j=&-\frac{(N+1)\lambda\eta c_m^1}{N}\begin{bmatrix}
            \vzero\\
            g_{\Delta,j}\\
            \vzero
        \end{bmatrix}
    \end{align}
    with norm bounds of
    \begin{equation}
        \max_{k,m,n}\norm{\mQ_{k,m,n}}\leq 1, \quad \max_{k,m,n}\norm{\mK_{k,m,n}}\leq 1, \quad \sum_{k,m,n}\norm{\mV_{k,m,n}}\leq 4\eta(C_1+C_3).
    \end{equation}
    Then, we have
    \begin{align}
        \hat{G}\coloneqq&\frac{1}{N+1}\sum_{j=1}^{N+1}\sum_{k,m,n}\sigma(\langle\mQ_{k,m,n}\vh_i,\mK_{k,m,n}\vh_j\rangle)\mV_{k,m,n}\vh_j \\
        =&-\frac{\eta}{n}\sum_{j=1}^n\begin{bmatrix}
            \vzero_{d+4} \\
            p(w_1g_{\Lambda,j},\vu_1^\top\vx_j)\cdot \vx_j \\
            \vdots \\
            p(w_Kg_{\Lambda,j},\vu_1^\top\vx_j)\cdot \vx_j \\
            \bar{r}(\vu_1^\top\vx_j)\cdot g_{\Lambda,j}\\
            \vdots \\
            \bar{r}(\vu_K^\top\vx_j)\cdot g_{\Lambda,j}\\
            0 \\
            \vdots \\
            0 \\
            \vzero_{4}
        \end{bmatrix}
        -\frac{\eta}{N}\sum_{j=1}^N\begin{bmatrix}
            \vzero_{d+4} \\
            -\lambda p(v_1 g_{\Delta,j},\vu_1^\top \vx_j)\cdot \vx_j \\
            \vdots \\
            -\lambda p(v_K g_{\Delta,j},\vu_1^\top \vx_j)\cdot \vx_j \\
            0 \\
            \vdots \\
            0 \\
            \bar{r}(\vu_1^\top\vx_j)\cdot g_{\Delta, j}\\
            \vdots \\
            \bar{r}(\vu_K^\top\vx_j)\cdot g_{\Delta, j} \\
            \vzero_{4}
        \end{bmatrix}.
    \end{align}
    For $\Lambda$, as $\abs{p(s, t)-s\cdot r'(t)}\leq \epsilon_p$, we have
    \begin{align}
        \abs{p(\vw_kg_{\Lambda,j},\vu_k^\top\vx_j)-\partial \gamma(\Lambda(\vx_j), y_j)\cdot w_k\cdot r'(\vu_k^\top \vx_j)}
        &\leq \epsilon_p+\abs{g_{\Lambda,j}-\partial_1\gamma(\Lambda(\vx_j), y_j)}\cdot\abs{w_k}\cdot\abs{r'(\vu_k^\top\vx_j)} \\
        &\leq \epsilon_p+B_uL_r(\epsilon_\gamma L_\gamma\epsilon_r),
    \end{align}
    where $L_r=\max\abs{r'(t)}$.
    Also,
    \begin{equation}
        \abs{\bar{r}(\vu_k^\top\vx_j)\cdot g_{\Lambda,j}-r(\vu_k^\top\vx_j)\cdot\gamma(\Lambda(\vx_j),y_j)}\leq B_g\epsilon_r+B_r(\epsilon_\gamma+B_uL_\gamma\epsilon_r).
    \end{equation}
    The same holds for $\Delta$.
    Consequently, letting the actual gradient corresponding to $\hat{G}$ be $G$, we have
    \begin{equation}
        \frac{1}{\eta}\norm{\hat{G}-G}_2\leq \sqrt{K}B_x(\epsilon_p+B_uL_r(\epsilon_\gamma+B_uL_\gamma\epsilon_r))+2\sqrt{K}(B_g\epsilon_r+B_r(\epsilon_\gamma+B_uL_\gamma\epsilon_r)).
    \end{equation}
    By choosing $\epsilon_p, \epsilon_\gamma, \epsilon_r$ appropriately, we can ensure $\norm{\hat{G}-G}\leq \eta\epsilon$.
    Thus, we have an attention block $\vh''\mapsto\vh'''_i=[\vx_i, y_i, t_i, s_i, 1, \vu-\eta\nabla_{\vu}(L-\lambda\Omega), \vw-\eta\nabla_\vw L, \vv-\eta\nabla_\vv\Omega, \vzero, \bar{\Lambda}(\vx_i), \bar{\Delta}(\vx_i), g_{\Lambda,i}, g_{\Delta,i}]$.
    
    \item \textbf{(The MLP block of the second layer)} This block projects the weights onto $\gW$ and fill the last four elements of $\vh'''_i$ with zeros, that is, $\vh'''_i\mapsto[\vx_i, y_i, t_i, s_i, 1, \Pi_\gW(\vu'), \Pi_\gW(\vw'), \Pi_\gW(\vv'), \vzero_4]$.
    By assumption, such a block exists.
\end{enumerate}
\end{proof}

\subsection{In-context UDA Algorithm Selection}\label{ap:sec:proof_icuda}

\begin{lemma} \label{lem:density_estimation}
    Suppose that the density $p_S$ is Lipschitz continuous.
    Define a density estimator $\hat{p}_S(\cdot) := (nh)^{-1} \sum_{i=1}^{n'} K'(\vx, \vx_i^T)$ with a bandwidth $h > 0$ and some kernel function $K'$, which is continuous and satisfying $\inf K' d = 1$.
    Fix $\varepsilon > 0$.
    For $n'  \geq \varepsilon^{-1/3} \log n$, with probability at least $1 - 1/n$, we obtain
    \begin{align}
        \sup_{\vx \in \mathcal{X}}| p_S(\vx) - \hat{p}_S(\vx)| \leq \varepsilon.
    \end{align}
\end{lemma}
\begin{proof}[Proof of Lemma \ref{lem:density_estimation}]
    We regard the kernel function $K'$ as parameterized by the bandwidth parameter $h$ as $K'_h$, such that $\hat{p}_S(\cdot) = (nh)^{-1} \sum_{i=1}^{n} K'_h(\vx, \vx_i)$ as a kernel function specified in \cite{jiang2017uniform}.
    We apply Theorem 2 in \cite{jiang2017uniform}, then we obtain the statement.
\end{proof}

\begin{proof}[Proof of \cref{thm:algorithm_selection}]
We show the theorem in the following steps.
We use a three-layer Transformer: the first self-attention block constructs a kernel density approximation $p_i=\hat{p}_S(\vx_i)$, the succeeding MLP block computes $\exp(-p_i)$, then the next self-attention block computes $\sum_i\exp(-\beta p_i)$, and then, the MLP layer computes $-\beta\sum_i\exp(-\beta p_i)$, and the attention block of the third Transformer layer selects the final output from the outputs of UDA algorithms based on $-\beta\log\sum_i\exp(-\beta p_i)>\delta$.
\begin{enumerate}[label=\textbf{Step \arabic*:},leftmargin=*,itemindent=30pt]
    \item \textbf{(The self-attention block of the first layer)}
    This block constructs kernel density estimation of the source domain $\hat{p}_S(\cdot)=\frac{1}{n}\sum_{i=1}^n K(\cdot, \vx_i)$, where $K$ is a kernel function, and obtains $p_i=\hat{p}_S(\vx_i)$ for $i=1,\dots N$.
    We adopt the RBF kernel as $K$, as it is universal.
    As the RBF kernel $(\vx, \vx')\mapsto K(\vx, \vx')$ is $(\epsilon_1, B_x, M_1, C_1)$-approximable, there exists a function $\bar{K}$ such that
    \begin{equation}
        \bar{K}(\vx, \vx')=\sum_{i=1}^{M_1} c_m \sigma(\va_m^\top\vx+\vb_m^\top\vx'+d_m), \quad\text{with } \sum_{m=1}^{M_1} |c_m|\leq C_1, |\va_m|+|\vb_m|+d_m\leq 1~(\forall m),
    \end{equation}
    where $\sup_{(\vx, \vx')\in [-B_x, B_x]^d\times [-B_x, B_x]^d}|K(\vx, \vx')-\bar{K}(\vx, \vx')|\leq \epsilon_1$.
    A self-attention layer can represent $\bar{K}$ with matrices
    \begin{align}
        \mQ_m\vh_i = \begin{bmatrix}
            \va_m \\
            \vx_i \\
            d_m \\
            -2B_x \\
            \vzero_{D-2(d+1)}
        \end{bmatrix}, \quad
        \mK_m\vh_j = \begin{bmatrix}
            \vx_j \\
            \vb_m \\
            1 \\
            1 - t_j \\
            \vzero_{D-2(d+1)}
        \end{bmatrix}, \quad
        \mV_m\vh_j = c_m\frac{N+1}{n}\begin{bmatrix}
            \vzero_{D-5} \\
            1 \\
            \vzero_4
        \end{bmatrix}.
    \end{align}
    These matrices exist, and we obtain
    \begin{align}
        &\frac{1}{N+1}\sum_{j=1}^{N+1}\sum_{m=1}^{M_1}\sigma(\langle\mQ_m\vh_i,\mK_m\vh_j\rangle)\mV_m\vh_j \\
        =&\frac{1}{n}\sum_{j=1}^{N+1}\sum_{m=1}^{M_1} c_m\sigma(\va_m^\top\vx_i+\vb_m^\top\vx_j+d_m-2B_x(1-t_j))[\vzero_{D-5}, 1, \vzero_4] \\
        =&\frac{1}{n}\sum_{j=1}^n \bar{K}(\vx_i, \vx_j)[\vzero_{D-5}, 1, \vzero_4].
    \end{align}
    Letting $p_i=\frac{1}{n}\sum_{j=1}^n\bar{K}(\vx_i, \vx_j)$, the output of this layer is $\vh^{(1)}_i=[\vz_i, \vgamma, \tilde{f}^{\mathrm{IWL}}(\vx_i), \tilde{f}^{\mathrm{DANN}}(\vx_i), \vzero_5]$ to $\vh'^{(1)}_i=[\vz_i, \vgamma, \tilde{f}^{\mathrm{IWL}}(\vx_i), \tilde{f}^{\mathrm{DANN}}(\vx_i), p_i, \vzero_4]$.

    \item \textbf{(The MLP block of the first layer)}
    Because $p\mapsto\exp(-\beta p)$ is $(\epsilon_2, 1, M_2, C_2)$-approximable, there exists a function $e(t)$, where 
    \begin{equation}
        e(t)=\sum_{m=1}^{M_2} c_m\sigma(a_m t+b_m) \quad\text{with } \sum_{m=1}^{M_2}|c_m|\leq C_2, |a_m|+|b_m|\leq 1~(\forall m),
    \end{equation}
    such that $\sup_{t\in [-R_2, R_2]} |e(t)-\exp(-\beta t)|\leq \epsilon_2$.
    By choosing matrices $\mW_1, \mW_2$ such that $\mW_1\in\R^{M_2\times D}$ such that
    \begin{equation}
        \mW_1\vh'_i = [a_m p_i + b_m]_{m=1}^{M_2},
    \end{equation}
    and $\mW_2\in\R^{D\times M_2}$ consisting of elements of
    \begin{equation}
        (\mW_2)_{j,m}=c_m\1(j=D-4),
    \end{equation}
    so that
    \begin{equation}
        \mW_2\sigma(\mW_1\vh'_i)=\sum_{m=1}^{M_2}\sigma(a_m p_i+ b_m) c_m \ve_{D-4}.
    \end{equation}
    Consequently, this block maps $\vh'^{(1)}_i=[\vz_i, \vgamma, \tilde{f}^{\mathrm{IWL}}(\vx_i), \tilde{f}^{\mathrm{DANN}}(\vx_i), p_i, \vzero_4]$ to $\tilde{\vh}^{(1)}_i=\vh^{(2)}_i=[\vz_i, \vgamma, \tilde{f}^{\mathrm{IWL}}(\vx_i), \tilde{f}^{\mathrm{DANN}}(\vx_i), p_i, e(p_i), \vzero_3]$.
    
    \item \textbf{(The Attention block of the second layer)}
    With a single-head self-attention block, we get $\sum_{i=n}^{n+n'}e(p_I)$ with the following matrices:
    \begin{align}
        \mQ_{1}\vh_i=\begin{bmatrix}
            1 \\
            1 \\
            \vzero_{D-2}
        \end{bmatrix}, \quad
        \mK_{1}\vh_j=\begin{bmatrix}
            e(p_i) \\
            1 - (t_j - s_j)\\
            \vzero_{D-2}
        \end{bmatrix}, \quad
        \mV_{1}\vh_j=(N+1)\begin{bmatrix}
            \vzero_{D-3} \\
            1 \\
            \vzero_2
        \end{bmatrix},
    \end{align}
    so that
    \begin{align}
        \sum_{j=1}^{N+1}\sigma(\langle\mQ_1\vh_i,\mK_1\vh_j\rangle)\mV_1\vh_j
        =\sum_{j=1}^{N+1}e(p_i)\1(n+1\leq j\leq n+n')[\vzero_{D-3}, 1, \vzero_2]
        =[\vzero_{D-3}, \sum_{j=n}^{n+n'}e(p_i), \vzero_2].
    \end{align}
    Thus, this block maps $\vh^{(2)}_i=[\vz_i, \vgamma, \tilde{f}^{\mathrm{IWL}}(\vx_i), \tilde{f}^{\mathrm{DANN}}(\vx_i), p_i, e(p_i), \vzero_3]$ to $\vh'^{(2)}_i=[\vz_i, \vgamma, \tilde{f}^{\mathrm{IWL}}(\vx_i), \tilde{f}^{\mathrm{DANN}}(\vx_i), p_i, e(p_i), \sum_{j=n}^{n+n'}e(p_j), \vzero_2]$
    
    \item \textbf{(The MLP block of the second layer)}
    Since $\log$ is $(\epsilon_4, n', M_4, C_4)$-approximable by sum-of-ReLUs, there exists a function $\mu(t)$, where
    \begin{equation}
        \mu(t)=\sum_{m=1}^{M_4}c_m\sigma(a_m t+b_m)\quad\text{with} \sum_{m=1}^{M_4}|c_m|\leq C_4, |a_m|+|b_m|\leq 1~(\forall m),
    \end{equation}
    such that $\sup_{t\in(0, n')}|\mu(t)-\log(t)|\leq \epsilon_4$.
    By selecting matrices $\mW_1, \mW_2$ such that $\mW_2\in\R^{1\times D}$ such that
    \begin{equation}
        \mW_1\vh_i = [a_m e+b_m]_{m=1}^{M_4},
    \end{equation}
    and $\mW_2\in\R^{D\times 1}$ consisting of elements of
    \begin{equation}
        (\mW_2)_{j,m}=-\frac{1}{\beta}\1(j=D),
    \end{equation}
    so that
    \begin{align}
        &\mW_2\sigma(\mW_1\vh_i)\\
        =&-\frac{1}{\beta}\sum_{m=1}^{M_4}c_m\sigma(a_m e+b_m)[\vzero, 1, 0]\\
        =&[\vzero, -\frac{1}{\beta}\mu( \sum_{j=n}^{n+n'} e(p_i)), 0].
    \end{align}
    As a result, by letting $q=-\frac{1}{\beta}\mu( \sum_{j=n}^{n+n'} e(p_i))$, this block maps $\vh'^{(2)}_i=[\vz_i, \vgamma, \tilde{f}^{\mathrm{IWL}}(\vx_i), \tilde{f}^{\mathrm{DANN}}(\vx_i), p_i, e(p_i), \sum_{j=n}^{n+n'}e(p_i), \vzero_2]$ to $\tilde{\vh}^{(2)}_i=\vh^{(3)}_i=[\vz_i, \vgamma, \tilde{f}^{\mathrm{IWL}}(\vx_i), \tilde{f}^{\mathrm{DANN}}(\vx_i), p_i, e(p_i), \sum_{j=n}^{n+n'}e(p_i), q, 0]$.
    
    \item \textbf{(The self-attention block of the third layer)}
    The outputs of UDA algorithms can be selected by $\1(q\geq \delta)\tilde{f}^{\mathrm{IWL}}(\vx_*)+(1-\1(q\geq \delta))\tilde{f}^{\mathrm{DANN}}(\vx_*)$.
    This process can be approximated by $\hat{\1}_a(t, s)=\sigma(a(t-s)+0.5)-\sigma(a(t-s)-0.5)$, where $a>0$, which approximates $\1(t\geq s)$ as $|\1(t\geq s)-\hat{\1}_a(t, s)|=\frac{1}{2a}$.
    This process can be represented by a self-attention layer with four heads consisting of the following matrices:
    \begin{align}
        \mQ_1\vh_i =& \begin{bmatrix}
            aq - a\delta + 0.5\\
            \vzero_{D-1}
        \end{bmatrix}, \quad
        \mK_1\vh_j =& \begin{bmatrix}
            1 \\
            \vzero_{D-1}
        \end{bmatrix}, \quad
        \mV_1\vh_j =& (N+1)\tilde{f}^{\mathrm{IWL}}(\vx_*)\ve_D, \\
        \mQ_2\vh_i =& \begin{bmatrix}
            aq - a\delta - 0.5\\
            \vzero_{D-1}
        \end{bmatrix}, \quad
        \mK_2\vh_j =& \begin{bmatrix}
            1 \\
            \vzero_{D-1}
        \end{bmatrix}, \quad
        \mV_2\vh_j =& -(N+1)\tilde{f}^{\mathrm{IWL}}(\vx_*)\ve_D, \\
        \mQ_3\vh_i =& \begin{bmatrix}
            aq - a\delta + 0.5\\
            \vzero_{D-1}
        \end{bmatrix}, \quad
        \mK_3\vh_j =& \begin{bmatrix}
            1 \\
            \vzero_{D-1}
        \end{bmatrix}, \quad
        \mV_3\vh_j =& (N+1)\tilde{f}^{\mathrm{DANN}}(\vx_*)\ve_D, \\
        \mQ_4\vh_i =& \begin{bmatrix}
            aq - a\delta - 0.5\\
            \vzero_{D-1}
        \end{bmatrix}, \quad
        \mK_4\vh_j =& \begin{bmatrix}
            1 \\
            \vzero_{D-1}
        \end{bmatrix}, \quad
        \mV_4\vh_j =& -(N+1)\tilde{f}^{\mathrm{DANN}}(\vx_*)\ve_D,
    \end{align}
    so that
    \begin{align}
        &\sum_{m=1,\dots,4}\sigma(\langle\mQ_m\vh_i, \mK_m\vh_j\rangle)\mV_m\vh_j \\
        =&\{\sigma(a(q-\delta)+0.5) - \sigma(a(q-\delta)-0.5)\}\tilde{f}^{\mathrm{IWL}}(\vx_*)\ve_D\\ \notag
        +&\{1-\sigma(a(q-\delta)+0.5) + \sigma(a(q-\delta)-0.5)\}\tilde{f}^{\mathrm{DANN}}(\vx_*)\ve_D \\
        =&\hat{\1}_a(q, \delta)\tilde{f}^{\mathrm{IWL}}(\vx_i)\ve_D+(1-\hat{\1}_a(q, \delta))\tilde{f}(\vx_i)^{\mathrm{DANN}}\ve_D.
    \end{align}
    Letting $\tilde{f}(\vx_i)=\hat{\1}_a(q, \delta)\tilde{f}^{\mathrm{IWL}}(\vx_i)\ve_D+(1-\hat{\1}_a(q, \delta))\tilde{f}(\vx_i)^{\mathrm{DANN}}\ve_D$, this block finally outputs $[\vz_i, \vgamma, \tilde{f}^{\mathrm{IWL}}(\vx_i), \tilde{f}^{\mathrm{DANN}}(\vx_i), p_i, e(p_i), \sum_{j=n}^{n+n'}e(p_i), q, \tilde{f}(\vx_i)]$.
\end{enumerate}

We have
\begin{align}
    &\norm{\tilde{f}-\hat{f}^{\mathrm{ICUDA}}}\\
    =&\norm{(\tilde{\zeta} \tilde{f}^{\mathrm{IWL}}+(1-\tilde{\zeta})\tilde{f}^{\mathrm{DANN}})-(\zeta \hat{f}^{\mathrm{IWL}}+(1-\zeta)\hat{f}^{\mathrm{DANN}})} \\
    \leq&\norm{\tilde{\zeta}}\norm{\tilde{f}^{\mathrm{IWL}}-\hat{f}^{\mathrm{IWL}}}+\norm{\tilde{\zeta}-\zeta}\norm{\hat{f}^{\mathrm{DANN}}-\hat{f}^{\mathrm{IWL}}}+\norm{1-\tilde{\zeta}}\norm{\tilde{f}^{\mathrm{DANN}}-\hat{f}^{\mathrm{DANN}}} \\
    \leq&\norm{\tilde{f}^{\mathrm{IWL}}-\hat{f}^{\mathrm{IWL}}} + C\norm{\tilde{\zeta}-\zeta} + \norm{\tilde{f}^{\mathrm{DANN}}-\hat{f}^{\mathrm{DANN}}}\\
    \leq& \epsilon_{\mathrm{IWL}} + \epsilon_{\mathrm{DANN}} + C\norm{\tilde{\zeta}-\zeta},
\end{align}
where $\tilde{\zeta}=\hat{\1}_a(q, \delta)$, $\zeta=\1(\min_i \hat{p}_S(\vx_i)>\delta)$, and $C=\norm{\tilde{f}^{\mathrm{DANN}}-\hat{f}^{\mathrm{DANN}}}$.
Further, defining $r=\softmin_\beta \hat{p}_S(\vx_i)$, we obtain
\begin{align}
    \norm{\tilde{\zeta}-\zeta} 
    &=\norm{\1(r>\delta)-\hat{\1}_a(q, \delta)} \\
    &\leq\norm{\1(r>\delta)-\hat{\1}_a(r, \delta)}+\norm{\hat{\1}_a(r, \delta)-\hat{\1}_a(q, \delta)} \\
    &\leq \frac{1}{2a}+\norm{\hat{\1}_a(r, \delta)-\hat{\1}_a(q, \delta)}
\end{align}
Assuming $r$ is random, and its density is bounded, the second term is $\norm{\hat{\1}_a(r, \delta)-\hat{\1}_a(q, \delta)}=0$ with probability of $1-O(|r-q|)$.
Here, we have
\begin{align}
    |r-q|&\leq |\softmin_\beta \hat{p}_S-\softmin_\beta p_S|B_x+|\softmin_\beta p_S-\min p_S|B_x \\
         &\leq \|\softmin_\beta\|\|\hat{p}_S-p_S\|B_x+\|\softmin_\beta-\min\|\|p_S\|B_x \\
         &\leq \frac{\log n}{\beta}\epsilon' B_x + \frac{\log n}{\beta} B_x,
\end{align}
where $\epsilon'=\|\hat{p}_S-p_S\|$.
We used $\softmin_\beta p_i = -\frac{1}{\beta}\log\sum_{i}^n\exp(- \beta p_i) \leq \frac{1}{\beta}\log n$ and $|\softmin_\beta p_i-\min p_i|\leq \frac{\log n}{\beta}$ for $(p_1, \dots, p_n)$ such that $p_i\geq 0$.
Hence, we can obtain $|r-q| < \epsilon$ with sufficiently large $\beta$.
Thus, by properly selecting $a, \beta, \epsilon_{\mathrm{IWL}}, \epsilon_{\mathrm{DANN}}$, $\norm{\tilde{f}-\hat{f}^{\mathrm{ICUDA}}}$ can be upper bounded by any $\epsilon>0$ with probability at least $1-1/n' - O(|r-q|) = 1-1/n' - O(\epsilon)$.
\end{proof}

\section{Experimental Details}\label{ap:sec:exp_settings}

For the experiments, we used a machine equipped with an AMD EPYC 7543P 32-core CPU and an NVIDIA RTX 6000 Ada GPU with CUDA 12.4.
The Transformer model and the baselines were implemented with PyTorch v2.3~\citep{pytorch}.

\subsection{Two-moon 2D}

In this problem, the source distribution is two-moon-shaped, and the target distribution is its random rotation, where the rotation is uniformly chosen between $-\pi/3$ and $\pi/3$.
Gaussian noise of a standard deviation of 0.1 is added to each data point.
During the pre-training of the Transformer model, the source distribution is randomly rotated.

We used a Transformer model with eight layers, eight heads, and an embed dimension of 128 without Layer Normalization and trained it with an AdamW optimizer~\citep{adamw} of a learning rate of $10^{-4}$ and a weight decay of $10^{-2}$ with a batch size of 64 for $10^4$ iterations.
The baseline models were optimized for $10^4$ iterations using an Adam optimizer~\citep{adam} with the full batch.
uLSIF was trained with a learning rate of $0.1$, and the regularization term $\lambda$ was set to $1.0$.
DANN was trained with a learning rate of $0.1$, and the regularization term $\lambda$ was set to $1.0$.
The source-only neural network follows the DANN's setup.

\subsection{Colorized MNIST}

In this setting, images of two digits are selected from the MNIST dataset to make a binary classification problem.
Images are downsampled to $8\times 8$ pixels.
Their background colors depend on domains, where the color of the source domain is a given RGB color, and that of the target domain is its complement.
During the pertaining phase, digits and colors are randomly chosen.

We trained the Transformer model with the same setting as the Two-moon 2D problem, except for the embed dimension set to 256.
The baseline models were optimized for $10^4$ iterations using an Adam optimizer with the full batch.
uLSIF was trained with a learning rate of $0.1$ and the regularization term $\lambda$ was set to $1.0$.
DANN was trained with a learning rate of $10^{-3}$, and the regularisation term $\lambda$ was set to $10$.
The source-only neural network follows the DANN's setup.

\end{document}